\documentclass[letterpaper, 10pt, conference]{ieeeconf}

\usepackage{graphics}
\usepackage{epsfig}

\usepackage{cite}
\usepackage[dvipsnames, svgnames]{xcolor}
\usepackage{xurl}
\usepackage{hyperref}
\hypersetup{
    colorlinks,
    linkcolor={red!50!black},
    citecolor={blue!50!black},
    urlcolor={blue!80!black}
}
\usepackage[normalem]{ulem}
\usepackage{bm}
\let\labelindent\relax
\usepackage{enumitem}
\usepackage{breqn}
\usepackage{amssymb}
\usepackage{tabularx}
\usepackage{booktabs}
\usepackage{dcolumn}
\newcolumntype{d}[1]{D..{#1}}
\newcolumntype{C}{>{\centering\arraybackslash}X}
\usepackage{soul}
\usepackage[linesnumbered,ruled,vlined, noend]{algorithm2e}

\usepackage{amsthm}

\newcommand{\reals}{\mathbb{R}}

\newcommand{\R}{\reals}

\newcommand{\Ccal}{\mathcal{C}}
\newcommand{\Dcal}{\mathcal{D}}
\newcommand{\Ecal}{\mathcal{E}}
\newcommand{\Fcal}{\mathcal{F}}
\newcommand{\Gcal}{\mathcal{G}}

\newcommand{\Pcal}{\mathcal{P}}

\newcommand{\Scal}{\mathcal{S}}
\newcommand{\Tcal}{\mathcal{T}}
\newcommand{\Ucal}{\mathcal{U}}
\newcommand{\Vcal}{\mathcal{V}}
\newcommand{\Wcal}{\mathcal{W}}
\newcommand{\Xcal}{\mathcal{X}}

\newcommand{\eqn}[1]{\begin{align} #1 \end{align}}

\newcommand{\beforefix}[1]{{\color{red}#1}}

\theoremstyle{plain}
\newtheorem{theorem}{Theorem}

\newtheorem{lemma}{Lemma}
\theoremstyle{definition}
\newtheorem{definition}{Definition}

\theoremstyle{remark}

\definecolor{yellow}{cmyk}{0.0,0.10,0.95,0.0}
\definecolor{pred}{cmyk}{0,0.8,0.70,0.0}
\definecolor{bluedefined}{cmyk}{0.46, 0.10, 0, 0.0}

\usepackage{makecell}
\usepackage[table]{xcolor}

\usepackage[most]{tcolorbox}
\usepackage{xcolor}

\definecolor{problemblue}{RGB}{239,247,252}
\definecolor{problemline}{RGB}{210,220,235}

\definecolor{problemgray}{RGB}{247,247,247}
\definecolor{problemgrayline}{RGB}{210,210,210}

\newtcolorbox{highlight}{
  enhanced,
  breakable,
  colback=problemgray,
  colframe=problemgray,
  boxrule=0pt,
  frame hidden,
  borderline west={3pt}{0pt}{problemgrayline},
  arc=0pt,
  outer arc=0pt,
  boxsep=0pt,
  left=7pt,
  right=6pt,
  top=6pt,
  bottom=6pt
}

\newtcolorbox{banner}{
  enhanced,
  breakable,
  colback=white,
  colframe=white,
  boxrule=0pt,
  frame hidden,
  borderline west={3pt}{0pt}{problemgrayline},
  arc=0pt,
  outer arc=0pt,
  boxsep=0pt,
  left=7pt,
  right=0pt,
  top=0pt,
  bottom=0pt
}

\definecolor{problemred}{RGB}{255,245,245}
\definecolor{problemredline}{RGB}{220,160,160}

\newtcolorbox{highlight_red}{
  enhanced,
  breakable,
  colback=problemred,
  colframe=problemred,
  boxrule=0pt,
  frame hidden,
  borderline west={3pt}{0pt}{problemredline},
  arc=0pt,
  outer arc=0pt,
  boxsep=0pt,
  left=7pt,
  right=6pt,
  top=6pt,
  bottom=6pt
}

\definecolor{problemblue}{RGB}{239,247,252}
\definecolor{problemline}{RGB}{210,220,235}
\newtcolorbox{highlight_blue}{
  enhanced,
  breakable,
  colback=problemblue,
  colframe=problemblue,
  boxrule=0pt,
  frame hidden,
  borderline west={3pt}{0pt}{problemline},
  arc=0pt,
  outer arc=0pt,
  boxsep=0pt,
  left=7pt,
  right=6pt,
  top=6pt,
  bottom=6pt
}
\usepackage{etoc}
\usepackage{titletoc}

\definecolor{proofblue}{RGB}{247,251,254}
\definecolor{proofline}{RGB}{225,233,242}

\newtcolorbox{highlight_blue_proof}{
  enhanced,
  breakable,
  colback=proofblue,
  colframe=proofblue,
  boxrule=0pt,
  frame hidden,
  borderline west={3pt}{0pt}{proofline},
  arc=0pt,
  outer arc=0pt,
  boxsep=0pt,
  left=7pt,
  right=6pt,
  top=6pt,
  bottom=6pt
}

\IEEEoverridecommandlockouts  
\usepackage[english]{babel}
\usepackage{footnotehyper}

\usepackage{hyperref,tablefootnote,footnotehyper}
\usepackage{amsmath,booktabs,mwe,pdflscape}
\hypersetup{colorlinks}

\usepackage{leftindex} 
\usepackage{graphicx} 
\usepackage{amsmath}
\usepackage{amssymb}
\usepackage{amsthm}
\usepackage{amsfonts}
\usepackage{float}
\usepackage{mathrsfs}
\usepackage{booktabs}
\usepackage{array}
\usepackage[nodisplayskipstretch]{setspace}
\makeatletter
\def\BState{\State\hskip-\ALG@thistlm}
\makeatother
\usepackage[capitalize,nameinlink,noabbrev]{cleveref} 
\theoremstyle{plain}

\theoremstyle{plain}
\newtheorem{prob}{Problem}
\crefname{prob}{Problem}{Problems}
\Crefname{prob}{Problem}{Problems}
\theoremstyle{definition}
\newtheorem{assumption}{Assumption}
\newtheorem{remark}{Remark}
\newtheorem{prop}{Proposition}
\crefname{assumption}{Assumption}{Assumptions}
\Crefname{assumption}{Assumption}{Assumptions}
\newcolumntype{M}[1]{>{\centering\arraybackslash}m{#1}}
\usepackage{pifont}
\usepackage{multirow}
\usepackage{tabularray}
\usepackage{cite}
\usepackage{algpseudocode}

\newcommand{\RNum}[1]{\uppercase\expandafter{\romannumeral #1\relax}}
\usepackage{bigints}

\newcommand{\gatekeeper}{\texttt{gatekeeper}}

\newcommand{\dualgatekeeper}{\texttt{Dual-gatekeeper}}
\algrenewcommand\textproc{}

\usepackage[font=small,skip=10pt]{caption}

\algrenewcommand\algorithmicrequire{\textbf{Input:}}
\algrenewcommand\algorithmicensure{\textbf{Output:}}

\hypersetup{
  colorlinks,
  citecolor=teal, 
  linkcolor=teal,
  urlcolor=purple}

\author{Kaleb Ben Naveed$^{1}$, Manveer Singh$^{1}$, Devansh R. Agrawal$^{1}$, and Dimitra Panagou$^{1,2}$%
\thanks{$^{*}$The authors would like to acknowledge the support of the National Science Foundation (NSF) under grant no. 2223845 and grant no. 1942907.}
\thanks{$^{1}$Department of Robotics, University of Michigan, Ann Arbor, MI, 48109 USA. 
{\tt\small \{kbnaveed@umich.edu\}}}
\thanks{$^{2}$Department of Aerospace Engineering, University of Michigan, Ann Arbor, MI, 48109 USA. }}

\title{\LARGE \bf
Trajectory Planning for Safe Dual Control with Active Exploration
} 
\usepackage[normalem]{ulem}
\def\arraystretch{1.2}

\newcommand\footnoteref[1]{\protected@xdef\@thefnmark{\ref{#1}}\@footnotemark}

\begin{document}

\maketitle
\thispagestyle{empty}
\pagestyle{empty}

\begin{abstract}
Planning safe trajectories under model uncertainty is a fundamental challenge.
Robust planning ensures safety by considering worst-case realizations of the uncertainty with fixed uncertainty bounds, which leads to conservative behavior.
Actively reducing uncertainty on-the-fly while planning for the system trajectories defines the dual control problem. Most approaches add a weighted exploration term to the mission cost, tuned to tradeoff the nominal objective and the uncertainty reduction, but without formal consideration of when exploration is beneficial to the mission. In addition, safety constraints are not in general considered in dual control approaches. 
In this paper we study a budget-constrained dual control problem, in which uncertainty is reduced while accounting for both safety constraints and a budget on the exploratory cost. To solve this problem we propose \dualgatekeeper{}, a framework that integrates the nominal mission objective with active exploration under formal guarantees of safety and budget feasibility, so that exploratory actions are allowed only when safety and budget feasibility are then formally verified. We develop deterministic and probabilistic safety instantiations of the proposed approach, demonstrated through quadrotor navigation with uncertain aerodynamic parameters and autonomous racing with uncertain tire--road friction coefficients, respectively.
\href{https://kalebbennaveed.github.io/dual-gtk/}{[Project Page]}\footnote{\scriptsize\url{https://kalebbennaveed.github.io/dual-gtk/}}
\href{https://youtu.be/l-pxMYxoITU}{[Video]}\footnote{\scriptsize\url{https://youtu.be/l-pxMYxoITU}}


\end{abstract}



\section{Introduction}
Planning safe trajectories in the presence of model uncertainty is a fundamental challenge in control and robotics \cite{knaup2024adaptive, xue2024learning, lopez2019dynamic_robust2, parwana2025foresee, lew2023risk_robust_3, sasfi2023robust_robust_1}. Consider the problem of planning or tracking a reference trajectory \cite{hanover2024autonomous, xue2024learning, al2025lla, Krinner-RSS-24} for a nonlinear system under constraints, while the system dynamics depend on uncertain parameters (e.g., aerodynamic drag or tire-road friction). If uncertainties are not properly taken into account, the resulting trajectories may violate safety constraints. Robust control methods including tube-based MPC \cite{lopez2019dynamic_robust2, agrawal2024gatekeeper}, sampling-based approaches \cite{blackmore2010probabilistic, lew2023risk_robust_3}, control barrier function (CBF) methods \cite{9482751, 9683085}, and contraction-based approaches \cite{7989693, sasfi2023robust_robust_1} account for model uncertainty when enforcing safety constraints, yet under fixed bounds without  determining when or how to explore to reduce uncertainty. Notably, even if safety is preserved, conservative uncertainty bounds can still degrade mission performance. Adaptive methods update uncertainty bounds using online measurements to reduce conservatism~\cite{lopez2019adaptive,sasfi2023robust_robust_1}, but do not explicitly determine when exploration would benefit mission performance.

\begin{figure} [t]
  \centering
\includegraphics[width=1.0\columnwidth]{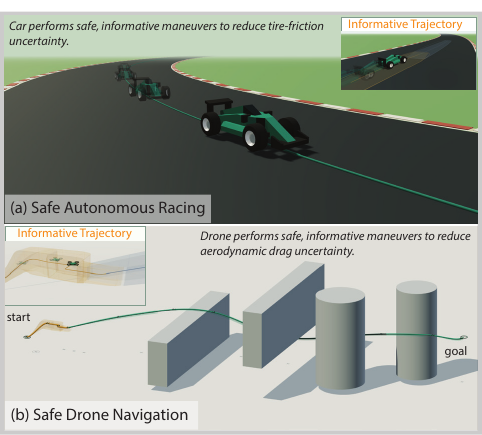}
  \caption{Illustration of (a) autonomous racing and (b) quadrotor navigation, with informative maneuvers highlighted in the insets.}
  \vspace{-20pt}
  \label{fig:car_and_drone}
\end{figure}

The \emph{dual control} problem \cite{feldbaum1961theory, parsi2023dual, mesbah2018stochastic, baltussen2026dual_np, mulagaleti2026dual_np, baltussen2025online_np_passive} refers to designing a controller that simultaneously minimizes a nominal mission cost and reduces uncertainty in unknown but learnable model parameters. Some approaches explicitly encourage exploration by augmenting the mission control objective with information-seeking or uncertainty-reduction terms, thereby generating trajectories that are informative about unknown parameters \cite{parsi2023dual, luo2024act}. Others rely on passive uncertainty reduction, i.e., updating unknown parameters using data collected during nominal trajectories \cite{sasfi2023robust_robust_1, Barcelos-RSS-21}. However, nominal trajectories may not generate sufficiently informative data, while weighted objectives can be sensitive to the chosen weight and may continue encouraging uncertainty reduction even when it no longer improves mission performance.

In this work, we study a \emph{budget-constrained variant of the dual control problem}, in which the system must (i) satisfy state and input constraints under model uncertainty, using either robust (deterministic) safety conditions or sampling-based probabilistic safety verification, (ii) actively reduce uncertainty through exploration, and (iii) ensure that the cumulative mission cost remains within a prescribed budget, representing the maximum allowable degradation in mission performance due to the exploration cost. In other words, exploration is desirable only if it can be certified to satisfy the safety and mission budget requirements. The probabilistic guarantee arises from sampling-based safety verification under bounded uncertainty, rather than stochastic dynamics.

Our proposed solution, \dualgatekeeper{}, is inspired by the
\gatekeeper{} framework~\cite{agrawal2024gatekeeper},
which provides runtime safety guarantees for system trajectories. \dualgatekeeper{} extends the earlier work by integrating safety with budget-feasible active exploration. The high-level idea is as follows: At each replanning time, pairs of candidate informative policies (for uncertainty reduction) and backup policies (for mission continuation without active exploration) are generated and evaluated for safety. The candidate
pairs that are safe (also called valid) are then checked for budget feasibility. Among the budget-feasible
pairs, the informative policy with the highest predicted
uncertainty-reduction score is selected; otherwise, the system
follows the previously committed (non-exploratory) policy, that is safe and does not induce further budget reduction. Building on
our conference version~\cite{naveed2025formal}, we develop deterministic (tube-MPC) and probabilistic safety instantiations,
demonstrated via extensive case studies on quadrotor navigation
with uncertain aerodynamic drag, and on autonomous racing with
uncertain tire-road friction, respectively, as illustrated in
Fig.~\ref{fig:car_and_drone}.

\section{Related Work and Contributions}

\subsection{Robust Safety Under Model Uncertainty}
Robust planning and control methods 
construct trajectories or feedback policies that guarantee constraint satisfaction despite uncertain dynamics, disturbances, or model mismatch. Tube-based model predictive control (MPC) methods enforce safety by tightening constraints and maintaining the closed-loop state within a robust tube around a nominal trajectory; examples include dynamic tube MPC and related extensions \cite{lopez2019dynamic_robust2, sasfi2023robust_robust_1, compton2025dynamic}. Sampling-based approaches provide an alternative route to robustness by approximating chance or risk constraints using scenario-based formulations such as particle control or sample average approximation (SAA) \cite{blackmore2010probabilistic, lew2024sample, lew2023risk_robust_3}. Safety can also be enforced using control barrier function (CBF) methods, which enforce forward invariance of safe sets under uncertainty and disturbances \cite{9482751, 9683085, knoedler2025safety}. Complementary contraction-based approaches establish
incremental stability through contraction metrics to certify the existence of safe feedback policies \cite{7989693, sasfi2023robust_robust_1}. However, they typically treat uncertainty as fixed, without considering how exploration can reduce uncertainty and improve subsequent mission performance.

\subsection{Dual Control and Active Uncertainty Reduction}
A complementary perspective is offered by the dual control problem, which explicitly balances exploitation (achieving the control objective) with exploration (actively reducing uncertainty) \cite{feldbaum1961theory}. In many control problems, uncertainty can be divided into two components: \emph{aleatoric uncertainty}, which reflects irreducible disturbances or noise, and \emph{epistemic uncertainty}, which arises from unknown but learnable model parameters. Dual control focuses on the latter by selecting control actions that both accomplish the nominal mission objective and generate informative data that reduce uncertainty in the model parameters. A wide range of approaches have been proposed along these lines \cite{baltussen2026dual_np, mulagaleti2026dual_np, baltussen2025online_np_passive,lew2022safe, parsi2020active, Kim-RSS-23, davydov2025first, sasfi2023robust_robust_1, Barcelos-RSS-21, hu2024active, hibbard2023safely, prajapat2025safe, luo2024act, parsi2023dual, mesbah2018stochastic, li2022dual, arcari2020dual, knaup2024adaptive, johnson2025implicit, soloperto2020augmenting, Zhang-RSS-25}. Some methods decouple learning and control by first reducing model uncertainty through a dedicated exploration phase and subsequently computing robust control policies based on the learned model \cite{lew2022safe, Kim-RSS-23, davydov2025first}. Other approaches rely on \emph{passive} uncertainty reduction, where model parameters are updated using data collected while executing the nominal control policy~\cite{baltussen2025online_np_passive, Barcelos-RSS-21, lopez2019adaptive, xue2024learning}. In contrast to approaches that separate exploration and mission execution, some active learning methods generate informative trajectories while simultaneously pursuing the mission objective~\cite{parsi2020active, baltussen2026dual_np, mulagaleti2026dual_np, prajapat2025safe, luo2024act, parsi2023dual}.

Another line of work considers hybrid reinforcement learning methods that balance exploration and exploitation through optimism and pessimism, using optimistic components to encourage broader exploration and pessimistic components to improve value estimation and stabilize learning. These methods incorporate the optimism-pessimism tradeoff directly within the actor-critic learning process through coupled actors, critics, or update rules \cite{9932556, tasdighi2024exploring, moskovitz2021tactical}. However, since the exploration-performance tradeoff is embedded directly within the optimization objective or the learning update, exploration in effect is encouraged implicitly, either through weighted uncertainty-reduction terms or through coupled optimistic-pessimistic updates, rather than being treated as a separate decision. While this can promote exploratory behavior, it does not provide a principled mechanism for determining when exploration should be enacted so that it maintains safety and constraints on the mission cost.

This paper builds upon our conference work~\cite{naveed2025formal} and introduces the \texttt{Dual-gatekeeper} framework that guarantees safety with budget-feasible active exploration. Compared to the conference paper, here we consider two approaches for predicting uncertainty reduction, and implement them on two distinct case studies, the latter of which is not included in the conference paper and highlights a case study where backup policies are not available by construction, introducing a probabilistic safety verification approach along with the corresponding analysis. The current paper also includes all the proofs that were omitted from the conference version. 

The remainder of the paper presents the problem formulation in Section~\ref{sec:prob}, the framework in Section~\ref{sec:framework}, uncertainty-reduction prediction in Section~\ref{sec:uncertainty_pred}, and the deterministic and probabilistic instantiations in Sections~\ref{sec:tube_mpc} and~\ref{sec:gatekeeper_robust}.

\section{Preliminaries \& Problem Formulation}\label{sec:prob}
\subsection{Notation}
Let $\mathbb{R}$, $\mathbb{R}_{\geq0}$, and
$\mathbb{R}_{>0}$ denote the reals, nonnegative reals,
and positive reals, respectively;
$|\mathcal A|$ denote the cardinality of a finite set, $\mathbb I\{\cdot\}$ the indicator function,
$\oplus$ and $\ominus$ denote the Minkowski
sum and Pontryagin difference, respectively. We use $\operatorname{Unif}(\mathcal A)$ to denote
the uniform distribution over a set $\mathcal A$.

\subsection{System Model}
\label{sec:system_model}
Consider a class of nonlinear control-affine systems with parametric uncertainty and bounded additive disturbance:
\begin{align}
    \dot{x} &=
    f_{0}(x) + F(x)\theta_f + \big(g_{0}(x) + G(x)\theta_g
    \big)u + w(t), \label{eq:system_model}
\end{align}
where $x \in \Xcal \subset \R^n$ is the state, $u \in \Ucal \subset \R^m$ is the control input, $\theta_f \in \Theta_f \subset \R^{p_f}$ is the unknown drift parameter vector contained in a known compact set, and $\theta_g = \begin{bmatrix}\theta_{g,1} & \cdots & \theta_{g,m}\end{bmatrix} \in \Theta_g \subset \R^{p_g \times m}$ is the unknown input parameter matrix contained in a known compact set, with $\theta_{g,j} \in \R^{p_g}$ denoting its $j$-th column. The unknown parameters are assumed constant throughout the mission. The parameters are collected into the vector
$
\theta
=
\begin{bmatrix}
\theta_f^\top &
\theta_{g,1}^\top &
\cdots &
\theta_{g,m}^\top
\end{bmatrix}^{\top}
\in\Theta\subset\R^p,
$ where $p:=p_f+mp_g$. The functions $f_0: \R^n \to \R^n$ and $g_0: \R^n \to \R^{n \times m}$ are the known nominal drift and input maps, respectively. The functions $F: \R^n \to \R^{n \times p_f}$ and $G: \R^n \to \R^{n \times p_g}$ are the known drift and input regressors, respectively. Furthermore, we denote $f(x,\theta_f) = f_{0}(x) + F(x)\theta_f$ and $g(x,u,\theta_g) = \big(g_{0}(x) + G(x)\theta_g\big)u$. Let $\Scal \subseteq \Xcal$ denote the safe set.

\begin{assumption}
\label{asm:noise_derivatives}
The additive disturbance $w:[t_0,\infty)\to\R^n$
satisfies $w(t)\in\Wcal$ for all $t\geq t_0$, where
$\Wcal:=\{v\in\R^n:\|v\|\leq\overline w\}$
and $\overline w\geq0$ is known.
\end{assumption} 

\begin{assumption}
\label{asm:full_state}
The full system state $x(t)$ is perfectly observed
for all $t\geq t_0$.
\end{assumption}

\subsection{Linear in Parameter form \& Parameter Identification}
\label{sec:lip_form}
Under~\cref{asm:noise_derivatives} and~\cref{asm:full_state},
the system dynamics \eqref{eq:system_model} can be rearranged to obtain a regression model
that is linear in the unknown parameter vector~$\theta$.
Specifically, define the signal
$
z(t) := \dot{x}(t) - f_0(x(t)) - g_0(x(t))u(t).
$
Substituting \eqref{eq:system_model} yields
\begin{equation*}
z(t) = F(x(t))\theta_f + (G(x(t))\theta_g)u(t) + w(t).
\end{equation*}
Using the representation
$\theta_g =
\begin{bmatrix}
\theta_{g,1} & \cdots & \theta_{g,m}
\end{bmatrix}$,
the input-dependent term can be written as
$
(G(x)\theta_g)u
=
\sum_{j=1}^{m} u_j G(x)\theta_{g,j}.
$
Define the regressor matrix
\begin{equation*}
\label{eqn:regressor}
\Phi(x,u)
:=
\big[
F(x)\;\;
u_1G(x)\;\;
\cdots\;\;
u_mG(x)
\big]
\in \R^{n\times(p_f+mp_g)}.
\end{equation*}
Then the dynamics admit the linear-in-parameters (LIP) form
\eqn{
\label{eq:lip_form_PE}
z(t)=\Phi(x(t),u(t))\theta + w(t).
}
\subsubsection{Persistent Excitation}
Consider the interval $[t-\Delta,t]$, where
$\Delta>0$ and $t\geq t_0+\Delta$, and integrate
\eqref{eq:lip_form_PE} over $[t-\Delta,t]$ to define \cite{Cohen2023} 
\begin{equation*}
\begin{aligned}
Y(t)&:=x(t)-x(t-\Delta)\\
&\quad-\int_{t-\Delta}^{t}
[f_0(x(s))+g_0(x(s))u(s)]\,ds,\\
\mathcal F(t)&:=\int_{t-\Delta}^{t}\Phi(x(s),u(s))\,ds,\\
W_\Delta(t)&:=\int_{t-\Delta}^{t}w(s)\,ds.
\end{aligned}
\end{equation*}
Then the derivative-free regression is
\begin{equation}
\label{eq:integral_regression}
Y(t)=\mathcal F(t)\theta+W_\Delta(t),
\end{equation}
with $\|W_\Delta(t)\|_\infty\leq\Delta\overline w$.

\begin{definition}[Persistent Excitation {\cite[\textit{Def}.~4.3]{Cohen2023}}]
\label{def:PE}
Consider the regression model in \eqref{eq:integral_regression}, where
$\mathcal{F}(t)$ is the regressor. The regressor is said to be persistently exciting
if there exist constants $T>0$ and $c>0$ such that 
\begin{equation}
\int_t^{t+T}\mathcal{F}(\tau)^\top \mathcal{F}(\tau)\,d\tau \ge cI, \mbox{ for all } t\ge t_0+\Delta.
\end{equation}
\end{definition}

\subsubsection{Finite Excitation}
In contrast to PE, which requires excitation over every interval
$[t,t+T]$, FE evaluates excitation using regressors computed over a finite number of measurement windows. At discrete times
$\tau_j$, $j=1,\ldots,M_k$, where
$t_0+\Delta\leq\tau_j\leq t_k$ and $t_k$ is the $k$-th
identification update time, we evaluate the continuous-time
regression \eqref{eq:integral_regression} over the window
$[\tau_j-\Delta,\tau_j]$, obtaining
$Y_j:=Y(\tau_j)$,
$\mathcal F_j:=\mathcal F(\tau_j)$, and
$W_{\Delta,j}:=W_\Delta(\tau_j)$, such that
$Y_j=\mathcal F_j\theta+W_{\Delta,j}$.
\begin{definition}[Finite Excitation {\cite[\textit{Def}.~4.5]{Cohen2023}}]
\label{def:FE} 
Let $\mathcal H_k:=\{(Y_j,\mathcal F_j)\}_{j=1}^{M_k}$ denote the $M_k$ stored regression pairs available at $t_k$. The corresponding regression pairs satisfies FE if there exists $\lambda_{\mathrm{FE}}>0$ such that
\begin{equation}
\lambda_{\min}\!\left(
\sum_{j=1}^{M_k}\mathcal{F}_j^\top \mathcal{F}_j
\right)
\ge \lambda_{\mathrm{FE}},
\end{equation}
\end{definition}
where $\lambda_{\min}(\cdot)$ denotes the smallest eigenvalue of the $p\times p$ matrix $\sum_{j=1}^{M_k}\mathcal F_j^\top\mathcal F_j$.

\subsubsection{Set-Membership Identification (SMID)}
\label{sec:smid}
Set-membership identification (SMID) refines the parameter uncertainty set using the stored regression data $\mathcal H_k$~\cite{cohen2023robust}. Let $\Theta_0\subset\R^p$ be the known initial hyperrectangle containing the true parameter, and let $\Theta_k$ be its updated hyperrectangular bound at time $t_k$. Let $\theta^\star\in\Theta_0$ be the true parameter. For each stored pair, define the regression error $e_j:=Y_j-\mathcal F_j\theta^\star\in\R^n$. Choose $\epsilon\geq0$ such that $\|e_j\|_\infty\leq\epsilon$ for every stored pair, accounting for the integrated disturbance (bounded by $\Delta\overline w$) and any numerical integration error.\footnote{The bound $\epsilon$ is typically chosen from the known disturbance bound and numerical integration error, e.g., $\epsilon=\Delta\overline w+\epsilon_{\mathrm{num}}$, where $\epsilon_{\mathrm{num}}$ bounds the integration error.} For $k\geq1$, the set consistent with the data is
\begin{equation*}
\label{eq:smid_feasible}
\begin{aligned}
\Theta_k^{\mathrm{feas}}:=\bigl\{\theta\in\Theta_{k-1}:\;&
\|Y_j-\mathcal F_j\theta\|_\infty\le\epsilon, j=1,\ldots,M_k\bigr\}.
\end{aligned}
\end{equation*}
Each stored pair gives linear inequalities in $\theta$, defining an intersection of halfspaces in parameter space. The minimum and maximum feasible values of each parameter are obtained by solving linear programs. Let $\theta_i \in \R$ denote the $i$-th component of the stacked parameter vector $\theta \in \R^p$, where $p = p_f + mp_g$. For each $i = 1, \dots, p$, the lower and upper bounds on $\theta_i$ are:
\begin{equation*}
\theta_i^{k,-} =
\min_{\theta\in\Theta_k^{\mathrm{feas}}}\theta_i,
\quad
\theta_i^{k,+} =
\max_{\theta\in\Theta_k^{\mathrm{feas}}}\theta_i.
\end{equation*}
The updated uncertainty set is
\begin{equation}
\label{eq:smid_update}
\Theta_k =
[\theta_1^{k,-},\theta_1^{k,+}]
\times
\cdots
\times
[\theta_p^{k,-},\theta_p^{k,+}],
\end{equation}
which contains $\Theta_k^{\mathrm{feas}}$. In this work, we use the midpoint of the SMID
uncertainty set $\Theta_k$ as the parameter estimate
$\hat{\theta}_k$ for planning, with
$\hat{\theta}_{k,i}=(\theta_i^{k,-}+\theta_i^{k,+})/2$
for $i=1,\ldots,p$.

The following result guarantees that the true parameter remains inside the
identified set.

\begin{lemma}[{\cite[Lemma~7.3]{cohen2023robust}}]
\label{lem:SMID_consistency}
Suppose~\cref{asm:noise_derivatives} holds, the regression error is bounded by $\epsilon$, and the parameter sets $\{\Theta_k\}$
are generated using the SMID update using~\eqref{eq:smid_update}. Then the sets satisfy
$\Theta_{k+1} \subseteq \Theta_{k} \subseteq \Theta_0$
and the true parameter satisfies $\theta^\star \in \Theta_k$ for all $k\ge0$.
\end{lemma}

The bounded regression error ensures that SMID retains the true parameter within the uncertainty set, while PE and FE characterize excitation relevant to reducing its size. We quantify this reduction using directional widths.

\begin{definition}[\textbf{Width of Parameter Set}]
\label{def:width}
Let $\Theta_k\subset\mathbb{R}^p$ denote the parameter
uncertainty set at time $k$, and let
$\Dcal\subset\mathbb{R}^p$ be a nonempty finite set
of unit directions. For any direction $d \in \Dcal$, the width of $\Theta_k$ along $d$ is defined as
\eqn{
w_d(\Theta_k) =
\sup_{\theta \in \Theta_k} d^\top \theta
-
\inf_{\theta \in \Theta_k} d^\top \theta.
}
This quantity measures the extent of the parameter uncertainty set along direction $d$. In the scalar case ($p=1$) with $d=1$ and $\Theta_k = [\theta_{\min}, \theta_{\max}]$, the width reduces to $w_d(\Theta_k) = \theta_{\max} - \theta_{\min}$.
\end{definition}

\subsection{Problem Statement}
We consider mission planning under bounded parametric uncertainty and additive disturbances, subject to state and input constraints. Reducing parameter uncertainty can enable less conservative policies and lower mission cost. We therefore seek informative actions that satisfy safety requirements and a budget on cumulative predicted exploration cost.
To formulate the problem mathematically, we first introduce some definitions as follows.

\begin{definition}[\textbf{Trajectory}]
\label{def:traj}
Let $\Tcal=[t_0,t_{f}]\subset\R$.
Let $\Pi$ denote the set of admissible feedback policies $\pi: \R \times \Xcal \to \Ucal$.
For a given initial state $x_0$ and a fixed parameter estimate
$\hat{\theta}\in\Theta$, with corresponding components
$\hat{\theta}_f$ and $\hat{\theta}_g$, a nominal trajectory
induced by a policy $\pi\in\Pi$ is a pair
$\big(p_x:\Tcal\to\Xcal,\; p_u:\Tcal\to\Ucal\big)$ such that for all $t\in\Tcal$,
\begin{equation}
\label{eq:traj_dynamics}
\begin{aligned}
\dot p_x(t)
&= f\big(p_x(t),\hat\theta_f\big)
+ g\big(p_x(t),p_u(t),\hat\theta_g\big),\\
p_u(t)&=\pi(t,p_x(t)), \qquad p_x(t_0)=x_0.
\end{aligned}
\end{equation}
\end{definition}

\begin{figure}[t]
  \centering
  \includegraphics[width=1.0\columnwidth]{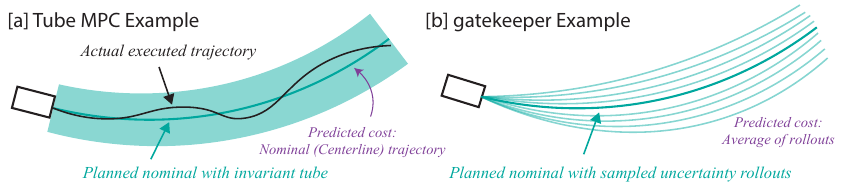}
  \caption{Predicted Cost Example}
  \vspace{-20pt}
  \label{fig:predicted_cost}
\end{figure}

\begin{definition}[Predicted Cost]
\label{def:mission_cost}
Let $\pi_k: \R \times \Xcal\to\Ucal$ denote a feedback policy at planning time
$t_k$, let $x_k:=x(t_k)\in\Xcal$ denote the state at $t_k$, and let
$\Tcal_k=[t_k,t_{k+1}]$ be a finite horizon of length $t_{k+1} - t_k>0$.
The predicted cost associated with $\pi_k$ is the functional
$\mathcal{J}^{k\to k+1}:\Pi\times\Xcal\times 2^{\R}\to\R_{\ge 0}$
\begin{equation*}
(\pi_k,x_k,\Tcal_k)\mapsto \mathcal{J}^{k\to k+1}(\pi_k,x_k,\Tcal_k),
\end{equation*}
which quantifies the \textit{predicted} performance of executing the policy $\pi_k$
from state $x_k$ over the horizon $\Tcal_k$.
\end{definition}

\subsubsection*{\textbf{Predicted Cost Examples}}

\textbf{(I)} In tube MPC \cite{lopez2019dynamic_robust2}, the robust policy $\pi^{\mathrm{rob}}_k$ consists of a finite-horizon nominal trajectory, an invariant tube, and ancillary feedback that maintains the closed-loop state within the tube. Its predicted cost $\mathcal{J}^{k\to k+1}_{\mathrm{rob}}$ is evaluated along the nominal tube centerline over $\Tcal_k=[t_k,t_{k+1}]$.
\textbf{(II)} In \gatekeeper{} \cite{agrawal2024gatekeeper, 11704677}, a candidate trajectory generated from a nominal planning policy undergoes safety verification. The candidate policy's predicted cost can be estimated as the maximum or average mission cost across $N$ forward rollouts over $\Tcal_k$ under sampled uncertainty realizations. These cost evaluations are illustrated in~\cref{fig:predicted_cost}.

\begin{definition}[Verified Safe Policies]
\label{def:safety verified}
At replanning time $t_k$, let
$x(t):=x(t;x_k,\pi_k,\theta,w)$ denote the closed-loop
trajectory of \eqref{eq:system_model} at time $t \geq t_k$, initiated at $x_k$ under the control law $u(t)=\pi_k(t,x(t))$,   parameter
$\theta$ and disturbance $w$. 
The set
$\Pi_k^{\mathrm{safe}}\subseteq\Pi$ is the set of policies $\pi_k$ satisfying either (i) or (ii): 

\textbf{(i) (Deterministic safety):}
\begin{align}
\label{eq:deterministic safety condition}
\Pi_k^{\mathrm{safe}}=\{&\pi_k \in \Pi \;|\; x(t; x_k, \pi_k, \theta, w)\in \mathcal S, u(t)\in \mathcal U, \nonumber\\ 
&\forall t\geq t_k, \forall \theta\in\Theta_k,\forall w(t)\in\Wcal\},
\end{align}

\textbf{(ii) (Probabilistic safety):}
\begin{align}
\label{eq:probabilistic safety condition}
\Pi_k^{\mathrm{safe}}=\{&\pi_k \in \Pi \;|\;\Pr_{(\theta,w)\sim{\mathbb P}_{k}}\bigl(
x(t)\in\Scal,\;
u(t)\in\Ucal,\;
\forall t\geq t_k
\bigr) \nonumber \\ 
&\geq 1-\varepsilon_k\},
\end{align}
where ${\mathbb P}_{k}
=
\mathbb{P}(\Theta_k, \Wcal)$ is the joint probability distribution of sampled parameters and disturbances, as explained later in~\cref{sec:gatekeeper_robust}, and $\varepsilon_k\in(0,1)$ is the prescribed
failure tolerance. In this work, we assume uniform distributions over
$\Theta_k$ and $\Wcal$ for probabilistic safety verification.
\end{definition}

\begin{figure*}[t]
  \centering
\includegraphics[width=2.00\columnwidth]{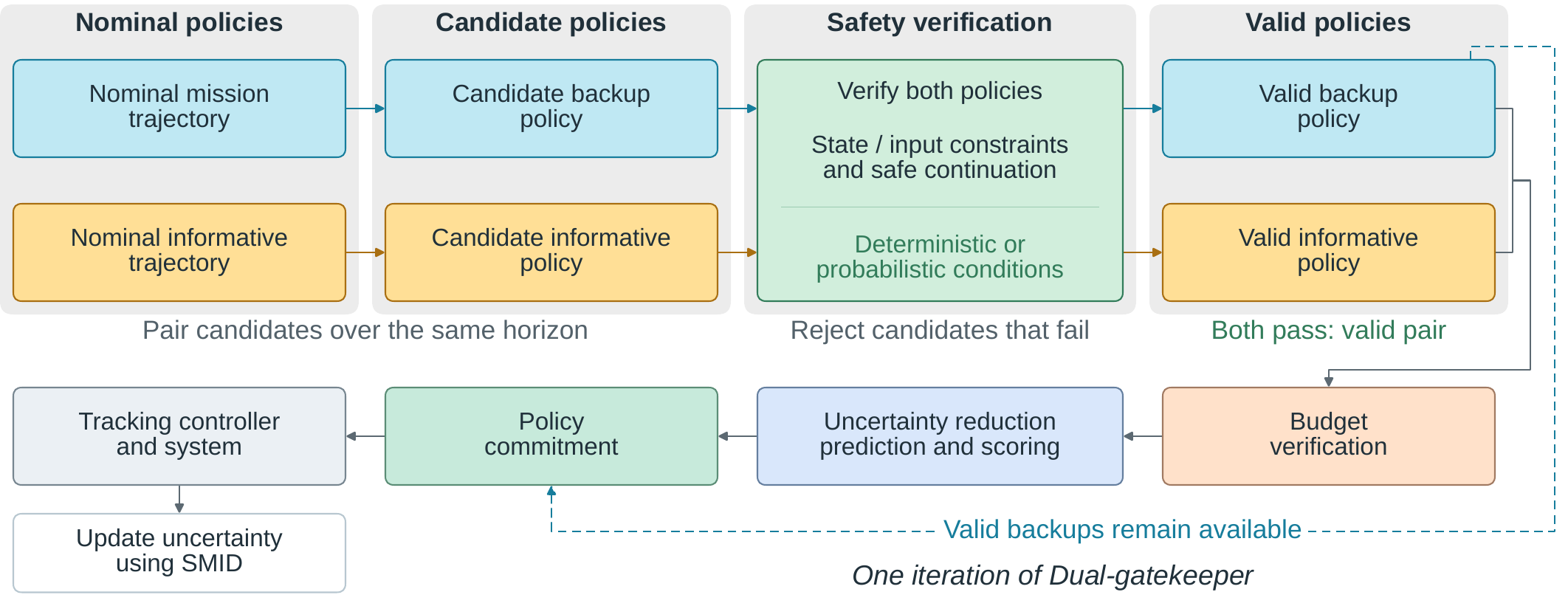}
  \caption{One iteration of \dualgatekeeper{}. Nominal mission and informative trajectories produce candidate backup and informative policies (\cref{subsec:candidate_generation}). Both policies in each pair are checked for safety; valid pairs undergo budget verification and uncertainty-reduction scoring (\cref{subsec:commitment,sec:uncertainty_pred}). A policy is then committed for execution and SMID updates the parameter uncertainty (\cref{sec:lip_form}).}
  \vspace{-13pt}
  \label{fig:method_block}
\end{figure*}

\begin{definition}[Nominal Mission Policy]
\label{def:nominal_policy}
At each replanning time $t_k$, let $x_k\in\Xcal$ be the state at time $t_k$,
$\hat\theta_k\in\Theta_k$ the parameter estimate from SMID as defined in~\cref{sec:smid}, and
$\ell:\Xcal\times\Ucal\to\R_{\geq0}$ the running mission cost.
Let $T_B>0$ denote the nominal planning horizon at $t_k$. The nominal trajectory
$p_k^{\mathrm{nom}}=(p_{k,x}^{\mathrm{nom}},p_{k,u}^{\mathrm{nom}})$
is obtained by solving
\begin{align}
\min_{p_x(\cdot),p_u(\cdot)}\quad
&\int_{t_k}^{t_k+T_B}
\ell(p_x(t),p_u(t))\,dt, \notag\\
\mathrm{s.t.}\quad
&p_x(t_k)=x_k,
\label{prob:nominal_generation}\\
&\dot p_x(t)=f(p_x(t),\hat{\theta}_{f,k})+g(p_x(t),p_u(t),\hat{\theta}_{g,k}), \notag\\
&p_x(t)\in\Scal,\quad p_u(t)\in\Ucal,\quad
\forall t\in[t_k,t_k+T_B]. \notag
\end{align}
Denote the nominal mission policy 
$\pi_k^{\mathrm{nom}}:\R\times\Xcal\to\Ucal$ as 
$p_{k,u}^{\mathrm{nom}}(t)
=\pi_k^{\mathrm{nom}}(t,p_{k,x}^{\mathrm{nom}}(t))$
for all $t\in[t_k,t_k+T_B]$.
By construction, the nominal mission policy minimizes nominal mission cost under the current uncertainty estimate realized via $\hat \theta$, i.e., not against all realizations of uncertainty.
\end{definition}

\begin{highlight}
\begin{prob}
\label{prob:overall_prob}
Consider the system \eqref{eq:system_model} with initial
state $x_0$ and initial parametric uncertainty set $\Theta_0$.
Let $t_0<t_1<\cdots<t_k<\cdots$ denote replanning times,
which are not necessarily uniformly spaced, and define $\Tcal_k=[t_k,t_{k+1}]$.
Let $B_{\mathrm{exp}}\geq0$ denote the exploration budget.
Given the nominal mission policy $\pi_k^{\mathrm{nom}}$
in Def.~\ref{def:nominal_policy}, we seek a sequence of
solution policies $\{\pi_k^{\mathrm{sol}}\}_{k=0}^{\infty}$
that pursue the mission while allowing informative deviations.
The objective is to maximize cumulative uncertainty reduction
subject to safety and an exploration-cost budget:
\begin{subequations}
\label{eq:overall_prob}
\eqn{
\max_{\{\pi_k^{\mathrm{sol}}\}_{k=0}^{\infty}}
\quad &
\sum_{k=0}^{\infty}
\frac{1}{|\Dcal|}
\sum_{d\in\Dcal}
\Delta w_d(\Theta_k) \\
\text{s.t.}\quad
& \pi_k^{\mathrm{sol}} \in \Pi_k^{\mathrm{safe}},
\quad \forall k\ge 0,
\label{prob1:cons_safety}\\
& \sum_{k=0}^{\infty} \Delta \mathcal{J}^{k}_{\mathrm{exp}}
\;\le\; B_{\mathrm{exp}},
\label{prob1:budget}
}
\end{subequations}
where the uncertainty width reduction along direction $d$ over the $k$-th replanning interval is defined as
\eqn{
\Delta w_d(\Theta_k)
=
w_d(\Theta_k)-w_d(\Theta_{k+1}),
}
and the predicted exploration cost incurred over $\Tcal_k$ is
\eqn{
\Delta \mathcal{J}^{k}_{\mathrm{exp}}
\;\coloneqq\;
\max\!\Big\{
0,\;
\mathcal{J}_{\mathrm{sol}}^{k \to k+1}
-
\mathcal{J}_{\mathrm{nom}}^{k\to k+1}
\Big\}.
}
Here, $\mathcal J_{\mathrm{sol}}^{k\to k+1}
:=\mathcal J^{k\to k+1}(\pi_k^{\mathrm{sol}},x_k,\Tcal_k)$
and $\mathcal J_{\mathrm{nom}}^{k\to k+1}
:=\mathcal J^{k\to k+1}(\pi_k^{\mathrm{nom}},x_k,\Tcal_k)$
denote the predicted costs of the solution and nominal
mission policies, respectively, per Def.~\ref{def:mission_cost}. The predicted exploration cost $\Delta\mathcal J_{\mathrm{exp}}^k$
is their nonnegative difference over $\Tcal_k$.
\end{prob}
\end{highlight}

\begin{remark}
The exploration budget in \eqref{prob1:budget} constrains the cumulative \emph{predicted} exploration cost $\Delta \mathcal{J}^{k}_{\mathrm{exp}}$, rather than the realized additional cost, which depends on unknown parameter and disturbance realizations. Bounding the latter would require additional assumptions.
\end{remark}


\section{Proposed Solution Framework}
\label{sec:framework}

Figure~\ref{fig:method_block} illustrates one iteration of \dualgatekeeper{}. 
The framework operates in a receding-horizon manner,
with decisions made at discrete, not necessarily uniformly
spaced, replanning times, according to 
the following steps. At each replanning time:

\textbf{1)} The trajectory planner generates the nominal mission policy according to Def.~\ref{def:nominal_policy}.

\textbf{2)} For each candidate horizon, nominal informative trajectories are generated using \eqref{prob:informative_generation}, and candidate informative and backup policies are constructed in pairs (\cref{subsec:candidate_generation}).

\textbf{3)} Both policies in each candidate pair are verified for safety according to Def.~\ref{def:safety verified}. A pair is valid if both policies satisfy the safety conditions. Valid pairs are then checked for budget feasibility using \eqref{eq:feasible-set}
(\cref{subsec:commitment}).

\textbf{4)} Among valid, budget-feasible pairs, the informative policy with the highest discounted uncertainty-reduction score is selected. Otherwise, a valid backup with the smallest time horizon is selected (if available) or the previously committed policy continues (\cref{subsec:commitment}).

\textbf{5)} The selected policy is executed, new measurements update the uncertainty set through SMID, and planning repeats. The definitions of the main concepts in the aforementioned algorithmic steps are given below.




\subsection{Candidate Policy Pair Generation}
\label{subsec:candidate_generation}
At each replanning time $t_k$, candidate backup
and informative policies are generated in pairs.
The backup candidates are constructed using the nominal
mission trajectory $p_k^{\mathrm{nom}}$ or its associated
policy $\pi_k^{\mathrm{nom}}$
(Def.~\ref{def:nominal_policy}), whereas informative
candidates are constructed from separately generated
nominal informative trajectories obtained as explained later on, in
\eqref{prob:informative_generation}.
The candidate policies are paired over common horizons
generated in increments of $T_c>0$, up to $T_B$:
\begin{equation}
\label{eq:set_horizons}
\Tcal_k^c
=
\big\{
T_i^{c,k}:
T_i^{c,k}=\min\{iT_c,T_B\}
\big\}_{i=1}^{N_k},
\end{equation}
Here $N_k=\lceil T_B/T_c\rceil$. Each duration defines the candidate interval $\Tcal_i^{c,k}:=[t_k,t_k+T_i^{c,k}]$ and switching time $t_{k,i}^s:=t_k+T_i^{c,k}$.



\subsubsection{Candidate Backup Policy}

For each candidate horizon $T_i^{c,k}$, let the candidate
backup policy be
\begin{equation}
\label{eq:backup_construction}
\bar{\pi}_{k,i}^{B}
=
\mathcal{B}_{k,i}
\big(p_k^{\mathrm{nom}},\Theta_k\big),
\end{equation} where $p_k^{\mathrm{nom}}$ is the nominal trajectory, $\Theta_k$ is the current uncertainty set, and $\mathcal{B}_{k,i}$ is the mapping that synthesizes the backup policy using the nominal policy $p_k^{\mathrm{nom}}$ and the current uncertainty set $\Theta_k$ over the candidate horizon
$T_i^{c,k}$. The constructions for tube MPC and the sampling-based instantiation are detailed in \cref{sec:tube_mpc,sec:gatekeeper_robust}, respectively.

\begin{definition}[\textbf{Valid Backup Policy}]
\label{def:backup_policy}
A candidate backup policy is called a valid backup policy, denoted $\pi_{k,i}^{B}$, for the system
\eqref{eq:system_model} if $\pi_{k,i}^{B} \in \Pi_k^{\mathrm{safe}}$ according to Def. \ref{def:safety verified}.
\end{definition}

\subsubsection{Candidate Informative Policy}

For each candidate horizon $T_i^{c,k}$, an informative
trajectory is generated to reduce uncertainty in the
unknown parameter $\theta$ while accounting for the
mission cost. Let
$\bar{p}_{k,i}^{I}
=(\bar{p}_{k,x,i}^{I},\bar{p}_{k,u,i}^{I})$
denote the nominal informative trajectory over
$\Tcal_i^{c,k}$ under the current parameter estimate $\hat \theta_k$,
obtained by solving
\begin{align}
\min_{p_x(\cdot),p_u(\cdot)}
\quad &
\int_{t_k}^{t_k+T_i^{c,k}}
\ell\big(p_x(\tau),p_u(\tau)\big)\,d\tau
\nonumber\\[-1mm]
&
-\gamma\log\det\big(\Gcal_{k,i}+\eta I\big),
\nonumber\\
\mathrm{s.t.}\quad
&
p_x(t_k)=x_k,
\label{prob:informative_generation}\\
&
\dot p_x(t)
=
f\big(p_x(t),\hat{\theta}_{f,k}\big)
+
g\big(p_x(t),p_u(t),\hat{\theta}_{g,k}\big),
\nonumber\\
&
p_x(t)\in\Xcal,\quad
p_u(t)\in\Ucal,\quad
\forall t\in\Tcal_i^{c,k}.
\nonumber
\end{align}
The first term penalizes nominal mission cost, and the second rewards excitation of the unknown parameters. The finite-horizon information Gramian $\Gcal_{k,i}\in\R^{p\times p}$ is
\begin{equation*}
\Gcal_{k,i}
:=
\int_{t_k}^{t_k+T_i^{c,k}}
\Phi\big(p_x(\tau),p_u(\tau)\big)^\top
W_\theta
\Phi\big(p_x(\tau),p_u(\tau)\big)\,d\tau,
\end{equation*}
where $\Phi(x,u)$ is the parameter regressor defined in
\eqref{eqn:regressor}, and
$W_\theta\in\mathbb{S}_{++}^{n}$ is a user-selected
weighting matrix. The scalar $\gamma>0$ weights information seeking, while
$\eta>0$ ensures the logarithm is well defined.
The log-determinant term promotes excitation across
parameter directions, favoring a smaller uncertainty ellipsoid.

\begin{definition}[\textbf{Candidate Informative Policy}]
\label{def:informative_segment}
The informative segment associated with the nominal informative $\bar p_{k,i}^I$
is a feedback policy
$\bar\pi_{k,i}^{I,\mathrm{seg}}:
\Tcal_i^{c,k}\times\Xcal\to\Ucal$ satisfying
\[
\bar p_{k,u,i}^I(t)
=\bar\pi_{k,i}^{I,\mathrm{seg}}
\big(t,\bar p_{k,x,i}^I(t)\big), \mbox{ for all }t\in[t_k,t_k+T_i^{c,k}].
\]
The complete candidate informative policy is defined as
\begin{equation}
\bar{\pi}_{k,i}^{I}(t,x)=
\begin{cases}
\bar{\pi}_{k,i}^{I,\mathrm{seg}}(t,x),
& t_k\leq t<t_{k,i}^{s},\\
\bar{\pi}_{k,i}^{B}(t,x),
& t\geq t_{k,i}^{s},
\end{cases}
\end{equation}
where $\bar{\pi}_{k,i}^{B}$ is the paired candidate
backup policy from \eqref{eq:backup_construction}.
The complete policy is subject to safety verification
before acceptance. The specific constructions are detailed in \cref{sec:tube_mpc,sec:gatekeeper_robust}. 
\end{definition}


\begin{definition}[\textbf{Valid Informative Policy}]
\label{def:valid_informative}
A candidate informative policy
$\bar{\pi}_{k,i}^{I}$
is called a valid informative policy, denoted by
$\pi_{k,i}^{I}$, for the system
\eqref{eq:system_model} if $\pi_{k,i}^{I} \in \Pi_k^{\mathrm{safe}}$, where $\Pi_k^{\mathrm{safe}}$ is understood according to Def. \ref{def:safety verified}.
\end{definition}

\begin{definition}[\textbf{Candidate Policy Pair}]
\label{def:candidate_pair}
The $i$-th candidate policy pair generated at
time $t_k$ is defined as $\big(\bar{\pi}_{k,i}^{I},
     \bar{\pi}_{k,i}^{B}\big),$ where both candidate policies are associated with
the common horizon $\Tcal_i^{c,k}$.
\end{definition}

\begin{definition}[\textbf{Valid Candidate Policy Pair}]
\label{def:valid_pair}
A candidate pair, denoted $(\pi_{k,i}^I,\pi_{k,i}^B)$, is valid if the candidate informative policy and the candidate backup policy satisfy
Defs.~\ref{def:valid_informative} and \ref{def:backup_policy}, respectively.
Let $\Vcal_k^c$ denote the set of all such valid pairs.
\end{definition}

\begin{algorithm}[!tbp]
\caption{\dualgatekeeper{} Framework}
\label{alg:dual_gatekeeper}
\footnotesize
\SetAlgoLined

\KwIn{Initial $t_0,x_0,\Theta_0$; budget $B_{\mathrm{exp}}$;
candidate increment $T_c$; directions $\Dcal$;
planning and verification parameters.}

\KwOut{Committed policies $\{\pi_k^{\mathrm{com}}\}$.}

Initialize $k\leftarrow0$ and
$\mathcal J_{\mathrm{exec}}^0\leftarrow0$\;

\While{the mission is not complete}{

  $\Vcal_k^c\leftarrow\emptyset$,
  $\mathcal I_k^B\leftarrow\emptyset$\;

  Attempt to generate $\pi_k^{\mathrm{nom}}$
  using \eqref{prob:nominal_generation}\;

  \If{nominal planning succeeds}{

    Generate candidate horizons using
    \eqref{eq:set_horizons}\;

    Generate nominal informative trajectories using
    \eqref{prob:informative_generation}; then construct
    the informative and backup candidate policies
    (\eqref{eq:backup_construction})\;

    Verify both policies and their continuations;
    collect verified backups in $\mathcal I_k^B$
    and valid pairs in $\Vcal_k^c$\;
  }

  \If{$k=0$ and $\mathcal I_k^B=\emptyset$}{
    Report infeasibility and terminate\;
  }

  Form $\Fcal_k^c$ using \eqref{eq:feasible-set}\;

  \ForEach{$i\in\Fcal_k^c$}{
    Predict $\Delta\xi_i$ using
    Section~\ref{sec:uncertainty_pred} and compute
    $s_i^{c,k}$;
  }

  Select $\pi_k^{\mathrm{com}}$ and $\tau_k$
  using Def.~\ref{def:committed_policy}\;

  Set $\Delta\mathcal J_{\mathrm{exp}}^k$ to the
  selected informative increment, or zero otherwise\;

  $\mathcal J_{\mathrm{exec}}^{k+1}
  \leftarrow\mathcal J_{\mathrm{exec}}^k
  +\Delta\mathcal J_{\mathrm{exp}}^k$\;

  Execute $\pi_k^{\mathrm{com}}$ for $\tau_k$
  or until mission completion\;

  $t_{k+1}\leftarrow$ end of execution\;

  \If{the mission continues}{
    Observe $x_{k+1}$ and update
    $\Theta_{k+1}$ using SMID\;
  }

  $k\leftarrow k+1$\;
}
\end{algorithm}

\subsection{Committing a Candidate Policy Segment}
\label{subsec:commitment}

Having defined the set of valid candidate policy pairs
$\Vcal_k^c$, the final step at replanning time $t_k$
is to decide which policy pair to commit for execution.

\subsubsection{Predicted Uncertainty Reduction} Let $\hat w_{d,k,i}$ denote the predicted width of the parameter set $\Theta_k$ along direction $d$ after executing candidate $i$. The corresponding mean directional-width reduction is
\begin{equation}
\label{eq:avg-width-reduction}
\Delta\xi_i
=
\frac{1}{|\Dcal|}
\sum_{d\in\Dcal}
\big(w_d(\Theta_k)-\hat w_{d,k,i}\big).
\end{equation}

Based on the predicted
uncertainty reduction $\Delta\xi_i$ due to the
informative policy segment, each valid candidate policy pair
$\big(\pi_{k,i}^{I},\pi_{k,i}^{B}\big)\in\Vcal_k^c$
is assigned a score $s_i^{c,k}$, defined as
\begin{equation}
\label{eq:score}
s_i^{c,k}
=
\exp\big(-\lambda T_i^{c,k}\big)\Delta\xi_i,
\qquad \lambda>0,
\end{equation}
where the exponential discount prioritizes uncertainty
reduction achieved over shorter candidate horizons. The score favors candidates that reduce parameter uncertainty
early, allowing subsequent planning decisions to benefit from
tighter uncertainty bounds. A longer candidate is preferred
only if its additional predicted uncertainty reduction
compensates for the time discount, controlled by $\lambda$. Section~\ref{sec:uncertainty_pred} presents two approaches to predict $\hat w_{d,k,i}$


\subsubsection{Enforcing Budget Feasibility} 
For a valid pair $(\pi_{k,i}^{I},\pi_{k,i}^{B})$, budget feasibility requires that the accumulated exploration-budget plus the candidate's predicted additional cost not exceed $B_{\mathrm{exp}}$. Thus,
\begin{equation}
\label{eq:feasible-set}
\begin{aligned}
\Fcal_k^c
=
\Big\{
i:\;&
(\pi_{k,i}^{I},\pi_{k,i}^{B})\in\Vcal_k^c, 
\mathcal J_{\mathrm{exec}}^k
+
\Delta\mathcal J_{\mathrm{exp}}^k(i)
\leq B_{\mathrm{exp}}
\Big\},
\end{aligned}
\end{equation}
where $\mathcal J_{\mathrm{exec}}^k :=
\sum_{j=0}^{k-1}
\Delta\mathcal J_{\mathrm{exp}}^j$ is the accumulated
(past) exploration cost of the committed policy
segments up to replanning time $t_k$. For a candidate pair indexed by $i$ at time $t_k$, let $\mathcal J_I^{k,i}$ and $\mathcal J_B^{k,i}$
denote the predicted costs of the informative and of the 
backup policy segments, respectively, over the
common interval $\Tcal_i^{c,k}$.
Both costs are evaluated according to
Def.~\ref{def:mission_cost}. For each informative candidate $\pi_{k,i}^I$,
its predicted cost is compared with that of its paired
backup $\pi_{k,i}^B$ over $\Tcal_i^{c,k}$.
If no informative policy is selected, the additional
predicted exploration cost is zero. The predicted exploration cost associated with the 
candidate pair $i$ is then
\begin{equation}
\label{eq:exp-cost-local}
\Delta\mathcal J_{\mathrm{exp}}^k(i)
:=
\max\Big\{
0,\,
\mathcal J_I^{k,i}-\mathcal J_B^{k,i}
\Big\}.
\end{equation}

\subsubsection{Verifying Safety and Defining Committed Policy}
Let $\mathcal I_k^B$ contain the indices of verified backup policies (Def.~\ref{def:backup_policy}), and let $i_B:=\min\mathcal I_k^B$ when nonempty. If $\Fcal_k^c=\emptyset$ and $\mathcal I_k^B\neq\emptyset$,
the shortest valid backup policy is selected.
If both sets are empty at $k\geq1$, execution continues
under the previously committed policy. If $\Fcal_k^c\neq\emptyset$, select the highest-scoring candidate:
\begin{equation}
\label{eq:best-candidate}
i^\star\in\arg\max_{i\in\Fcal_k^c}s_i^{c,k},
\end{equation}
breaking ties by the smallest index. The resulting committed policy is defined as follows.

\begin{definition}[\textbf{Committed Policy}]
\label{def:committed_policy}
The committed policy is defined as:
\begin{equation}
\label{eq:committed}
\pi_k^{\mathrm{com}}
=
\begin{cases}
\pi_{k,i^\star}^{I},
&\Fcal_k^c\neq\emptyset,\\
\pi_{k,i_B}^{B},
&\Fcal_k^c=\emptyset,\ \mathcal I_k^B\neq\emptyset,\\
\pi_{k-1}^{\mathrm{com}},
& \text{otherwise, for } k\geq1. 
\end{cases}
\end{equation}
The execution duration $\tau_k$ is $T_{i^\star}^{c,k}$,
$T_{i_B}^{c,k}$, or $T_c$, respectively.
\end{definition}

Starting from $\mathcal J_{\mathrm{exec}}^0=0$, update
$\mathcal J_{\mathrm{exec}}^{k+1}
=\mathcal J_{\mathrm{exec}}^k
+\Delta\mathcal J_{\mathrm{exp}}^k$.
The feasibility test~\eqref{eq:feasible-set} ensures
$\mathcal J_{\mathrm{exec}}^{k+1}\leq B_{\mathrm{exp}}$
under either safety instantiation. Execute the policy until $t_{k+1}=t_k+\tau_k$ or mission
completion. At the next replanning time, update $\Theta_{k+1}$
using SMID and repeat.
The procedure is summarized in \cref{alg:dual_gatekeeper}. This process repeats until mission completion, with a new
committed policy selected at each replanning time whenever possible,
or the previously committed policy continued if no new policy
is verified.


\section{Prediction of Uncertainty Reduction}
\label{sec:uncertainty_pred}

At replanning time $t_k$, each valid and budget-feasible candidate policy pair $(\pi_{k,i}^{I},\pi_{k,i}^{B})$ is scored using the predicted uncertainty reduction $\Delta\xi_i$ from its informative policy. It measures the predicted reduction in average directional width in \eqref{eq:avg-width-reduction} and enters the score $s_i^{c,k}$ in \eqref{eq:score}. We present two alternative uncertainty-reduction prediction approaches.

The first is a simulation-based approach that predicts
uncertainty reduction through forward rollouts under sampled parameter and
disturbance realizations. The second is a data-consistency-based approach that
predicts reduction directly from the candidate trajectory under bounded
disturbances. Both approaches provide the directional-width predictions $\hat w_{d,k,i}$ used in
\eqref{eq:avg-width-reduction}, yielding the $\Delta\xi_i$ that is used to score and
select informative candidates in
Section~\ref{subsec:commitment}.

\subsection{\textbf{Approach 1}: Simulation-Based Prediction of Uncertainty Reduction}
To evaluate informative candidate policy segments, the framework predicts the
expected reduction in parametric uncertainty induced by executing a candidate
policy $\pi^{I}_{k,i}$ over the interval
$\Tcal_i^{c,k}=[t_k,t_k+T_i^{c,k}]$. This prediction is obtained via forward
simulation of the closed-loop system for a finite number of admissible parameter
and disturbance realizations.

Starting at $x_k$, generate $N$ independent rollouts. Each rollout $\ell=1,\ldots,N$ uses one parameter realization $\theta^{(\ell)}\sim\mathcal U(\Theta_k)$, held fixed throughout the rollout, and independent disturbance samples $w_j^{(\ell)}\sim\mathcal U(\Wcal)$, $j=1,\ldots,N_i$, where $N_i$ is the number of simulation steps over $T_i^{c,k}$. The dynamics \eqref{eq:system_model} are simulated under $\pi_{k,i}^{I}$. These uniform distributions specify the prediction model, rather than an assumption needed for SMID consistency (see \cref{alg:uncert_pred}).

Each simulated state--input history yields integrated regression pairs over complete windows of length $\Delta$ in $\Tcal_i^{c,k}$, as in \eqref{eq:integral_regression}. Applying SMID to this synthetic history $\mathcal H_i^{(\ell)}$ predicts the updated set $\Theta_{k+1,i}^{(\ell)}$. We use the same regression-error bound as SMID, e.g., $\epsilon=\Delta\overline w+\epsilon_{\mathrm{num}}$ when $\epsilon_{\mathrm{num}}$ bounds numerical integration error. If the rollout contains no complete window, the predicted set remains $\Theta_k$.

The predicted reduction for rollout $\ell$ is the decrease in the average directional width of its synthetic SMID set:
\begin{equation*}
\Delta \xi_i^{(\ell)}
=
\frac{1}{|\Dcal|}
\sum_{d \in \Dcal}
\Big(
w_d(\Theta_k)
-
w_d(\Theta_{k+1,i}^{(\ell)})
\Big),
\label{eq:rollout_width_reduction}
\end{equation*}
Here $w_d(\cdot)$ is the directional width in Def.~\ref{def:width}; the calculation is summarized in \cref{alg:uncert_pred}.

\begin{algorithm}[!tbp]
\caption{Uncertainty Reduction Prediction via Parallel Rollouts}
\label{alg:uncert_pred}
\footnotesize
\SetAlgoLined

\KwIn{
Candidate informative policy $\pi_{k,i}^{I}$ over
$\Tcal_i^{c,k}=[t_k,t_k+T_i^{c,k}]$; 
state $x_k$; uncertainty set $\Theta_k$; pointwise disturbance set $\Wcal$; 
SMID update rule; direction set $\Dcal$; number of rollouts $N$.}

\KwOut{Predicted uncertainty reduction $\Delta \xi_i$}

\BlankLine

Initialize rollout buffer $\{\Delta \xi_i^{(\ell)}\}_{\ell=1}^{N} \leftarrow 0$\;

\For{$\ell = 1,\dots,N$ \textbf{in parallel}}{

Sample $\theta^{(\ell)} \sim \mathcal{U}(\Theta_k)$\;

Sample disturbances $w^{(\ell)}_j \sim \mathcal{U}(\Wcal)$ for $j=1,\dots,N_i$\;

Forward simulate \eqref{eq:system_model} from $x_k$ under $\pi_{k,i}^{I}$ using $(\theta^{(\ell)},\{w^{(\ell)}_j\})$\;

Form $\mathcal H_i^{(\ell)}$ from complete integration
windows within the rollout\;

$\Theta_{k+1,i}^{(\ell)}
\leftarrow\mathrm{SMID}(\Theta_k,\mathcal H_i^{(\ell)})$\;

$\Delta \xi_i^{(\ell)} \leftarrow
\frac{1}{|\Dcal|}\sum_{d\in\Dcal}
\big(w_d(\Theta_k)-w_d(\Theta_{k+1,i}^{(\ell)})\big)$\;

}

$\Delta \xi_i \leftarrow \frac{1}{N}\sum_{\ell=1}^{N}\Delta \xi_i^{(\ell)}$\;
\end{algorithm}
Finally, the predicted uncertainty reduction associated with candidate $i$ is
defined as the empirical average over the $N$ independent rollouts,
\begin{equation*}
\Delta \xi_i
=
\frac{1}{N}
\sum_{\ell=1}^{N}
\Delta \xi_i^{(\ell)},
\label{eq:predicted_uncertainty_reduction}
\end{equation*}
which is used directly in the candidate scoring rule \eqref{eq:score}. 
Equivalently, this approach uses
\begin{equation*}
\hat w_{d,k,i}
=
\frac{1}{N}
\sum_{\ell=1}^{N}
w_d\big(\Theta_{k+1,i}^{(\ell)}\big).
\end{equation*}

\subsection{\textbf{Approach 2}: Data-Consistency–Based Prediction of Uncertainty Reduction}

This approach predicts uncertainty reduction directly from the nominal informative trajectory $\bar p_{k,i}^{I}$ over $\Tcal_i^{c,k}$, obtained from \eqref{prob:informative_generation}. Rather than simulating closed-loop execution, it uses the regressors along this planned trajectory to bound how much the SMID uncertainty set could shrink, before any new measurements are collected.

For candidate $i$, let $N_j$ be the number of integration windows used in the prediction, with endpoints $\tau_j$, $j=1,\ldots,N_j$. The integrated regressor $\mathcal F_j$ is computed along the planned informative trajectory, using recorded
state and input data for any portion of the integration window preceding $t_k$. The corresponding future regression data $Y_j\in\R^n$ satisfy
\begin{equation*}
Y_j
=
\mathcal F_j\theta+v_j,
\qquad
\|v_j\|_\infty\leq\epsilon,
\end{equation*}
Here $\epsilon$ is the SMID regression-error bound from \cref{sec:lip_form}. Although the future $Y_j$ are unknown, the width bound below depends only on the planned regressors and $\epsilon$.

Stacking the regressors gives the matrix $A$ that maps a parameter difference to the corresponding change in the integrated regression outputs:
\begin{equation}
\label{eq:Astack}
A
=
\begin{bmatrix}
\mathcal F_1^\top &
\cdots &
\mathcal F_{N_j}^\top
\end{bmatrix}^{\top}
\in\R^{M\times p},
\qquad
M=N_j n.
\end{equation}

Let $\theta^\star\in\Theta_k$ be the true parameter. Two parameter values $\theta,\theta^\star\in\Theta_k$ can explain the same stacked measurements under the error bound $\epsilon$ if and only if there exist $v,\tilde v\in\R^M$ such that
\begin{equation*}
A\theta+v=A\theta^\star+\tilde v,
\qquad
\|v\|_\infty,\|\tilde v\|_\infty\leq\epsilon.
\end{equation*}
The key question is: \textit{under what conditions can two distinct parameters $\theta$ and $\theta^\star$ produce the same measurements within the noise bound?}
If many such parameters remain feasible, the uncertainty set stays large; if few remain, the uncertainty reduces.
Hence, the fewer alternative parameters that fit the data within the noise bound, the greater the uncertainty reduction achieved by the planned trajectory.
\begin{remark}
The executed trajectory is unknown at planning time,
and accounting for all possible deviations from the planned
trajectory can lead to overly conservative uncertainty-reduction
predictions. We therefore use the nominal informative trajectory
$\bar p_{k,i}^{I}$ as a practical approximation to predict
uncertainty reduction before execution.
\end{remark} 

\begin{highlight_blue}
\begin{lemma}
\label{lemma:2w}
Let $\theta^\star\in\Theta_k$ be the true parameter and define
$e_\theta:=\theta-\theta^\star\in\R^p$.
For each integrated regressor $\mathcal F_j$,
the following equivalence holds:
\begin{equation}
\label{eq:single-iff}
\begin{aligned}
&\exists v_j\in\R^n:
\quad
\|v_j\|_\infty\leq\epsilon,
\\
&\qquad
\|\mathcal F_j(\theta^\star-\theta)+v_j\|_\infty
\leq\epsilon
\\
&\qquad\Longleftrightarrow
\|\mathcal F_j e_\theta\|_\infty
\leq2\epsilon.
\end{aligned}
\end{equation}
Applying this condition to all $N_j$ integration
windows yields
$\|Ae_\theta\|_\infty\leq2\epsilon$.
\end{lemma}
\end{highlight_blue}

\begin{figure*}[t]
  \centering
\includegraphics[width=2.05\columnwidth]{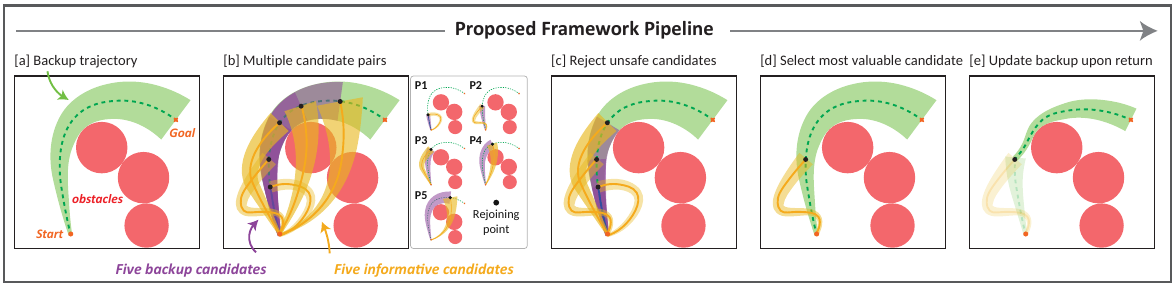}
  \caption{\small{Illustrative tube-MPC candidate selection (\cref{subsec:candidate_generation,subsec:commitment,sec:tube_mpc}). (a) Backup trajectory and tube; (b) paired informative and backup candidates; (c) rejection of unsafe pairs; (d) selection of a safe, budget-feasible candidate; (e) backup update after execution and identification. The candidate trajectories represent variations in state space, and their depiction in the physical environment is schematic.}}
  \vspace{-13pt}
  \label{fig:method_squence}
\end{figure*}

\begin{highlight_blue_proof}

\begin{proof}
If $v_j$ satisfies \eqref{eq:single-iff}, the triangle
inequality gives
\[
\|\mathcal F_j e_\theta\|_\infty
\leq\|\mathcal F_j e_\theta-v_j\|_\infty
+\|v_j\|_\infty\leq2\epsilon.
\]
Conversely, if $\|\mathcal F_j e_\theta\|_\infty\leq2\epsilon$,
choosing $v_j=\mathcal F_j e_\theta/2$ satisfies both
inequalities in \eqref{eq:single-iff}.
Applying this condition for each $j=1,\ldots,N_j$
and stacking the corresponding regressors as in
\eqref{eq:Astack} yields
$\|Ae_\theta\|_\infty\leq 2\epsilon$.
\end{proof}
\end{highlight_blue_proof}
By Lemma~\ref{lemma:2w}, an offset $e_\theta$ can remain indistinguishable under some admissible errors only if $\|Ae_\theta\|_\infty\leq2\epsilon$; offsets with $\|Ae_\theta\|_\infty>2\epsilon$ cannot explain the same data. We therefore define the set of potentially indistinguishable offsets.
\begin{definition}[Error set]\label{def:Ecal}
Given the stacked planned regressor $A$ in \eqref{eq:Astack} and the error bound $\epsilon$, define
\begin{equation*}
\Ecal_\theta \;:=\; \{\,e_\theta\in\R^p:\ \|A e_\theta\|_\infty \le 2\epsilon\,\}.
\label{eq:E-def}
\end{equation*}
\end{definition}
The set $\Ecal_\theta$ contains every parameter offset that could remain consistent with the same data as $\theta^\star$ for some admissible regression errors. Consequently, after collecting the planned measurements, every feasible parameter lies in the following outer bound:
For any admissible realized noise sequence, every
parameter consistent with the resulting measurements
must belong to $\theta^\star+\Ecal_\theta$.
The set $\Ecal_\theta$ characterizes parameter offsets that may remain indistinguishable from $\theta^\star$ under admissible regression errors. Hence, the parameter set consistent with the resulting measurements is contained in the following outer bound:
\begin{equation*}
\Theta_{N_j}(\theta^\star)
:=
\Theta_k\cap(\theta^\star+\Ecal_\theta).
\end{equation*}
This bound includes every parameter consistent with any admissible future error realization, though the feasible set for the measurements actually obtained may be smaller. For a direction $d\in\R^p$, its width is defined in Def.~\ref{def:width}; in particular, $d=e_i$, the $i$-th standard basis vector, gives the width of parameter component $i$. Thus,
\begin{equation*}
w_d \big(\Theta_{N_j}(\theta^\star)\big) 
= w_d \big(\Theta_k \cap (\theta^\star+\Ecal_\theta)\big).
\end{equation*}
Since intersection cannot increase width and $w_d(\theta^\star+\Ecal_\theta)=w_d(\Ecal_\theta)$ by translation invariance,
\begin{equation*}
w_d \big(\Theta_{N_j}(\theta^\star)\big) 
\;\le\; \min \!\big(w_d(\Theta_k),\, w_d(\Ecal_\theta)\big).
\end{equation*}
We next show how $w_d(\Ecal_\theta)$ can be computed from the planned trajectory given through planned regressor $A$.

\begin{highlight_blue}
\begin{lemma}\label{lem:width-2h}
 For any $d\in\R^p\!\setminus\!\{0\}$,
\eqn{
\label{eq:width-is-2h}
w_d(\Ecal_\theta)=2\,h_{\Ecal_\theta}(d),
\qquad
h_{\Ecal_\theta}(d)=\sup_{e_\theta \in \Ecal_{\theta}} d^\top e_\theta}
\end{lemma}
\end{highlight_blue}

\begin{highlight_blue_proof}
\begin{proof}
For any $e_\theta \in \mathcal{E}_\theta$,
\begin{equation*}
\|Ae_\theta\|_\infty = \|-Ae_\theta\|_\infty =
\|A(-e_\theta)\|_\infty   \le 2\epsilon.
\end{equation*}
Since $-e_\theta \in \mathcal{E}_\theta$ and $0 \in \mathcal{E}_\theta$, this implies $w_d(\Ecal_\theta) = \sup_{e_\theta \in \Ecal_\theta} d^\top e_\theta - \inf_{e_\theta \in \Ecal_\theta} d^\top e_\theta = 2\sup_{e_\theta \in \Ecal_\theta} d^\top e_\theta.$
\end{proof}
\end{highlight_blue_proof}
Let $\mathbf{1}_M\in\R^M$ be the all-ones vector. The condition $\|Ae_\theta\|_\infty\le2\epsilon$ is equivalent to $-2\epsilon\mathbf{1}_M\le Ae_\theta\le2\epsilon\mathbf{1}_M$, with inequalities interpreted componentwise. By \eqref{eq:width-is-2h}, the support function is computed by
\begin{equation}
\label{eq:lp_eq}
\begin{aligned}
\max_{e_\theta\in\R^p}\quad
& d^\top e_\theta
\\
\text{s.t.}\quad
&
-2\epsilon\mathbf{1}_M
\leq
Ae_\theta
\leq
2\epsilon\mathbf{1}_M.
\end{aligned}
\end{equation}
Let the multipliers
$\lambda_1,\lambda_2\in\R^M_{\geq0}$.
The Lagrangian of \eqref{eq:lp_eq} is
\begin{equation*}
\begin{aligned}
\mathcal L(e_\theta,\lambda_1,\lambda_2)
={}&
d^\top e_\theta
\\
&+
\lambda_1^\top
\big(2\epsilon\mathbf{1}_M+Ae_\theta\big)+
\lambda_2^\top
\big(2\epsilon\mathbf{1}_M-Ae_\theta\big).
\end{aligned}
\end{equation*}
Its supremum over $e_\theta$ is finite if and only if
$A^\top(\lambda_2-\lambda_1)=d$.

When $d\in\operatorname{range}(A^\top)$,
linear-programming duality gives
\begin{equation*}
\begin{aligned}
h_{\Ecal_\theta}(d)
=
\min_{\lambda_1,\lambda_2\in\R^M_{\geq0}}
\quad&
2\epsilon\mathbf{1}_M^\top
(\lambda_1+\lambda_2)
\\
\text{s.t.}\quad&
A^\top(\lambda_2-\lambda_1)=d.
\end{aligned}
\end{equation*}

Let $\lambda_d:=\lambda_2-\lambda_1$.
For a vector $v$, define
$(v)_+:=\max\{v,0\}$ and
$(v)_-:=\max\{-v,0\}$ componentwise.
For a given $\lambda_d$, the choice
$\lambda_2=(\lambda_d)_+$ and
$\lambda_1=(\lambda_d)_-$ minimizes
$\mathbf{1}_M^\top(\lambda_1+\lambda_2)$.
Hence,
\begin{equation}
\label{eq:prop_des_form}
h_{\Ecal_\theta}(d)
=
2\epsilon
\min_{\lambda_d\in\R^M:\,A^\top\lambda_d=d}
\|\lambda_d\|_1.
\end{equation}
If $d\notin\operatorname{range}(A^\top)$, the dual
constraint is infeasible and
$h_{\Ecal_\theta}(d)=+\infty$.
In this case, the width bound below reduces to the
prior width $w_d(\Theta_k)$, so this direction contributes
no guaranteed uncertainty reduction.
Therefore,
\eqn{
w_d\big(\Theta_{N_j}(\theta^\star)\big) 
\le \min\big(
w_d(\Theta_k),\; 2h_{\Ecal_\theta}(d)\big).
}

For candidate $i$, let $\Ecal_{\theta,k,i}$ be the error set in Def.~\ref{def:Ecal} formed using the regressors along $\bar p_{k,i}^{I}$. Combining \eqref{eq:width-is-2h} with \eqref{eq:prop_des_form}, we predict
\begin{equation*}
\hat w_{d,k,i}
=
\min\big\{
w_d(\Theta_k),
2h_{\Ecal_{\theta,k,i}}(d)
\big\}.
\end{equation*}

\begin{remark}[Conservative uncertainty prediction]
The bound requires the executed and planned regressors to
coincide and the prediction and SMID update to use the same
regression data and error bound. For hyperrectangular SMID,
it applies directly to $\Dcal=\{e_1,\ldots,e_p\}$ because
taking the smallest enclosing box preserves coordinate widths.
Other directions require accounting for the box approximation.
\end{remark}

\subsection{Choice of Prediction Approach}
Both approaches predict the uncertainty reduction of a candidate informative policy, but use different state--input trajectories. Approach~1 simulates its closed-loop execution under parameter realizations sampled from $\Theta_k$ and disturbances sampled from $\Wcal$. The simulated regressors therefore capture deviations from the nominal informative trajectory $\bar p_{k,i}^{I}$. These rollouts are more expensive but can be evaluated in parallel. Approach~2 instead uses the regressors along $\bar p_{k,i}^{I}$ and the bounded regression error $\epsilon$ to compute a width bound without forward simulation. This is less expensive, but the bound is valid only when executed and planned regressors coincide. Approach~1 is used for autonomous racing, whereas Approach~2 is used for quadrotor navigation.


\section{Deterministic Instantiation of \dualgatekeeper{}}
\label{sec:tube_mpc}
We consider the deterministic instantiation of \dualgatekeeper{} using robust tube MPC~\cite{lopez2019dynamic_robust2}. The backup policy combines the nominal control input with ancillary feedback, while a robust tracking tube bounds the closed-loop state deviation. Candidate informative policies are verified for robust state and input constraints and safe rejoining with the backup tube (Def.~\ref{def:safety verified}). Safety is therefore required for every $\theta\in\Theta_k$ and all admissible disturbance signals, without a probabilistic failure tolerance.

\subsection{Construction of Valid Candidate Backup Policy}
In this instantiation, tube MPC constructs a backup policy whose validity follows from constraint tightening and tube invariance, as shown below.

At replanning time $t_k$, tube MPC plans from $x_k$ to the terminal goal set at terminal time $t_f$, with horizon $T_B:=t_f-t_k$ as in Def.~\ref{def:nominal_policy}. Let $p_k^{\mathrm{nom}}=(p_{k,x}^{\mathrm{nom}},p_{k,u}^{\mathrm{nom}})$ denote the nominal mission trajectory computed under the current parameter estimate $\hat\theta_k$ and uncertainty set $\Theta_k$.

In addition to the nominal trajectory, tube MPC
uses a robust tracking tube with cross-section $\Ecal_k(t)\subseteq\R^n$ that bounds the deviation
between the nominal and actual closed-loop state
trajectories for all $\theta\in\Theta_k$ and all $w(t)\in\Wcal$ for every $t\in[t_k,t_f]$. The corresponding backup tube is
\begin{equation}
\label{eq:quad_backup_tube}
\Omega_k^B(t)
=
p_{k,x}^{\mathrm{nom}}(t)\oplus\Ecal_k(t).
\end{equation}

Let $\Delta\Ucal_k(t)\subseteq\R^m$ denote the set
of ancillary control corrections associated with the
tube. Robust constraint satisfaction is enforced using
the tightened sets
\begin{equation*}
\begin{aligned}
\overline{\Scal}_k(t)
&=
\Scal\ominus\Ecal_k(t), \quad
\overline{\Ucal}_k(t)
=
\Ucal\ominus\Delta\Ucal_k(t),
\\
\overline{\Xcal}_{\mathrm{goal},k}
&=
\Xcal_{\mathrm{goal}}\ominus\Ecal_k(t_f).
\end{aligned}
\end{equation*}
where $\Xcal_{\mathrm{goal}}\subseteq\Scal$
denotes the terminal goal set. Specifically, the tube-MPC nominal trajectory is computed by solving
\begin{align}
\min_{p_x(\cdot),p_u(\cdot)}
\quad &
\int_{t_k}^{t_f}
\ell\big(p_x(t),p_u(t)\big)\,dt
+\ell_T\big(p_x(t_f)\big),
\nonumber\\
\mathrm{s.t.}\quad
&
p_x(t_k)=x_k,
\label{prob:quad_backup}\\
&
\dot p_x(t)
=
f\big(p_x(t),\hat{\theta}_{f,k}\big)
+
g\big(p_x(t),p_u(t),\hat{\theta}_{g,k}\big),
\nonumber\\
&
p_x(t)\in\overline{\Scal}_k(t),
\quad
p_u(t)\in\overline{\Ucal}_k(t),
\nonumber\\
&
p_x(t_f)\in\overline{\Xcal}_{\mathrm{goal},k}.
\nonumber
\end{align}
Here, $\ell_T:\Xcal\to\R_{\geq0}$ denotes the terminal mission cost. Compared with the nominal planning problem \eqref{prob:nominal_generation}, \eqref{prob:quad_backup} uses tightened constraints, a terminal cost, and a terminal goal constraint. The dynamics and tightened constraints hold for all $t\in[t_k,t_f]$.

A sliding-mode ancillary feedback controller
$\pi_k^{\mathrm{arc}}:[t_k,t_f]\times\Xcal\to\mathbb{R}^m$
maintains the actual closed-loop state within the backup tube.
The corresponding backup policy is
\begin{equation}
\label{eq:quad_backup_policy}
\pi_k^B(t,x)
=
p_{k,u}^{\mathrm{nom}}(t)
+
\pi_k^{\mathrm{arc}}(t,x).
\end{equation}

The ancillary controller renders
$\Omega_k^B(t)$ robustly forward invariant over
$[t_k,t_f]$. In particular, for any
$t_s\in[t_k,t_f]$,
\begin{equation}
\label{eq:quad_backup_invariance}
\begin{aligned}
x(t_s)\in\Omega_k^B(t_s)
\Longrightarrow
x(t)\in\Omega_k^B(t),
\forall t\in[t_s,t_f],
\end{aligned}
\end{equation}
under $\pi_k^B$, for all $\theta\in\Theta_k$
and $w(t)\in\Wcal$. Together with the tightened constraints,
\eqref{eq:quad_backup_invariance} ensures
satisfaction of the original state and input
constraints over the backup horizon.

For every candidate duration $T_i^{c,k}$ in \eqref{eq:set_horizons}, the candidate backup policy is the same valid tube-MPC policy over the remaining backup horizon:
\begin{equation}
\label{eq:quad_candidate_backup}
\pi_{k,i}^B(t,x)=\pi_k^B(t,x),
\qquad t\in[t_k,t_f].
\end{equation}
The candidate duration $T_i^{c,k}$ specifies how long the policy is executed before replanning; the backup policy remains defined through $t_f$ to provide safe continuation.

\subsection{Construction of Robust Candidate Informative Policy}

For each candidate interval $\Tcal_i^{c,k}$ defined by \eqref{eq:set_horizons}, the valid backup policy $\pi_{k,i}^B$ remains available through $t_f$. We generate a nominal informative trajectory over $\Tcal_i^{c,k}$, then construct the candidate informative feedback policy and verify its robust safety and rejoining with the backup tube at $t_{k,i}^s$.

The nominal informative trajectory
$\bar{p}_{k,i}^I=
(\bar{p}_{k,x,i}^I,\bar{p}_{k,u,i}^I)$
is generated using
\eqref{prob:informative_generation}, with the
additional terminal rejoining constraint
\begin{equation}
\label{eq:quad_info_terminal}
\bar{p}_{k,x,i}^I(t_{k,i}^s)
=
p_{k,x}^{\mathrm{nom}}(t_{k,i}^s).
\end{equation}

The terminal constraint
\eqref{eq:quad_info_terminal}
requires the nominal informative trajectory to
rejoin the nominal backup centerline.
However, nominal terminal matching alone does not
ensure that the actual closed-loop state can safely
rejoin the backup tube. To account for deviations under model uncertainty
and disturbances, a robust tube is constructed
around the nominal informative trajectory:
\begin{equation}
\label{eq:quad_info_tube}
\Omega_{k,i}^I(t)
=
\bar{p}_{k,x,i}^I(t)
\oplus\Ecal_{k,i}^I(t),
\quad t\in\Tcal_i^{c,k},
\end{equation}
where $\Ecal_{k,i}^I(t)$ bounds the closed-loop
state deviation for all $\theta\in\Theta_k$ and
$w(t)\in\Wcal$.

The informative feedback segment
$\bar{\pi}_{k,i}^{I,\mathrm{seg}}$ employs ancillary feedback
to maintain the actual closed-loop state within
$\Omega_{k,i}^I(t)$. Let
$\Delta\Ucal_{k,i}^I(t)$ denote the associated
set of ancillary control corrections, so that
$u(t)\in\bar{p}_{k,u,i}^I(t)\oplus
\Delta\Ucal_{k,i}^I(t)$.

The following conditions provide a constructive verification of deterministic safety under Def.~\ref{def:safety verified}, including safe rejoining with the backup tube:

\textbf{1) Robust constraint satisfaction:}
The informative tube and nominal informative control input
must satisfy
\begin{subequations}
\label{eq:quad_info_safety}
\begin{align}
&
\Omega_{k,i}^I(t)\subseteq\Scal,
\label{eq:quad_info_state}\\
&
\bar{p}_{k,u,i}^I(t)
\in\Ucal\ominus\Delta\Ucal_{k,i}^I(t),
\label{eq:quad_info_input}
\end{align}
\end{subequations}
for all $t\in\Tcal_i^{c,k}$.

\textbf{2) Robust terminal rejoining:}
At the switching time $t_{k,i}^s$, the
informative tube must satisfy
\begin{equation}
\label{eq:quad_tube_rejoining}
\Omega_{k,i}^I(t_{k,i}^s)
\subseteq\Omega_k^B(t_{k,i}^s).
\end{equation}
Conditions \eqref{eq:quad_info_safety} ensure
state and input constraint satisfaction for all
$t\in\Tcal_i^{c,k}$, $\theta\in\Theta_k$, and
$w(t)\in\Wcal$.
Furthermore, \eqref{eq:quad_tube_rejoining}
ensures that the actual terminal state lies
within the backup tube. The backup policy
$\pi_k^B$ can therefore be safely applied
from $t_{k,i}^s$, with
\eqref{eq:quad_backup_invariance} guaranteeing constraint satisfaction for all
$t\in[t_{k,i}^s,t_f]$. The complete candidate informative policy is therefore
\begin{equation*}
\bar\pi_{k,i}^I(t,x)=
\begin{cases}
\bar\pi_{k,i}^{I,\mathrm{seg}}(t,x),
&t_k\leq t<t_{k,i}^s,\\
\pi_k^B(t,x),&t_{k,i}^s\leq t\leq t_f.
\end{cases}
\end{equation*}

A candidate informative policy satisfying
\eqref{eq:quad_info_safety} and
\eqref{eq:quad_tube_rejoining} is a valid
informative policy $\pi_{k,i}^I$ according to
Def.~\ref{def:valid_informative}.
Since the backup policy $\pi_{k,i}^B$ is valid
by construction, the pair
$(\pi_{k,i}^I,\pi_{k,i}^B)$ is valid according to
Def.~\ref{def:valid_pair}. Candidates that do
not satisfy these conditions are rejected.

For each valid candidate pair, the predicted costs
$\mathcal J_I^{k,i}$ and $\mathcal J_B^{k,i}$
are evaluated along the corresponding nominal tube
centerlines over the common candidate horizon.
The predicted uncertainty reduction $\Delta\xi_i$
is evaluated using Approach~2, the
data-consistency-based method introduced in
\cref{sec:uncertainty_pred}.

\begin{figure*}[t]
  \centering
  \includegraphics[width=2.05\columnwidth]{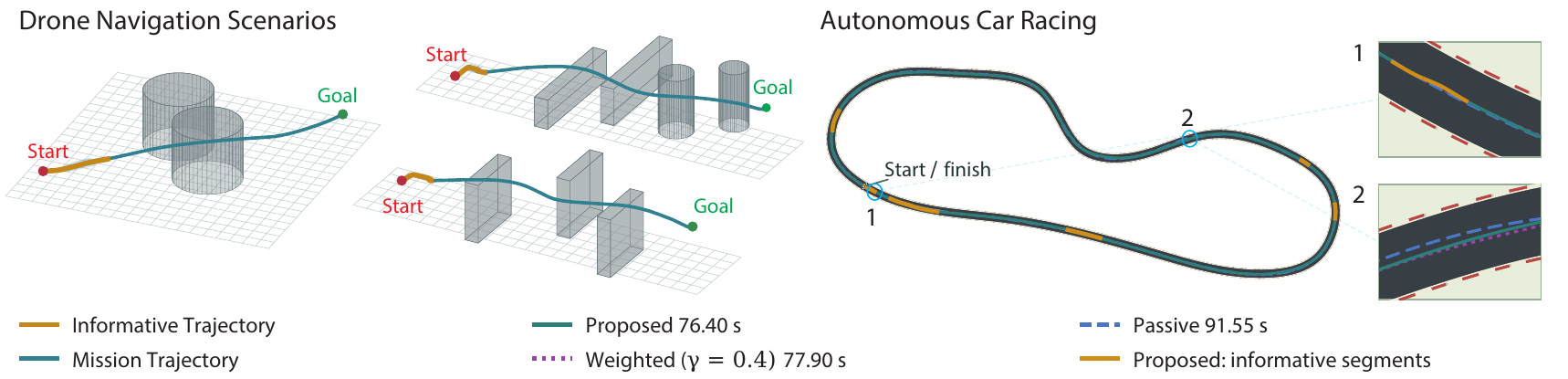}
  \caption{Comparisons for the deterministic and probabilistic instantiations of \dualgatekeeper{}. Left: Three quadrotor navigation scenarios showing informative segments (orange) and mission trajectories (teal). Right: Autonomous racing trajectories comparing the proposed method with passive and weighted-information baselines, with enlarged views of two selected track regions. ($\varepsilon=1\%$).}
  \vspace{-13pt}
  \label{fig:car_drone_scenarios}
\end{figure*}

\subsection{Deterministic Safety Guarantee}

We now establish the safety of the actual closed-loop
trajectory and budget feasibility under the sequence
of policy segments committed by \dualgatekeeper{}.

\begin{highlight_blue}
\begin{theorem}[Deterministic Safety and Budget Feasibility]
\label{thm:deterministic_safety}

Consider the system \eqref{eq:system_model}
under \cref{asm:noise_derivatives,asm:full_state}.
Suppose that the uncertainty sets generated by SMID
satisfy \cref{lem:SMID_consistency}, and that at
each replanning time $t_k$:

\textbf{(i)} A feasible backup policy $\pi_k^B$ from \eqref{eq:quad_backup_policy} is available, with $x_k\in\Omega_k^B(t_k)$, and satisfies \eqref{eq:quad_backup_invariance} and the tightened constraints in \eqref{prob:quad_backup}.

\textbf{(ii)} Whenever an informative policy $\pi_{k,i^\star}^I$ is selected for execution, its nominal input and informative tube satisfy \eqref{eq:quad_info_safety}, and its tube satisfies the rejoining condition \eqref{eq:quad_tube_rejoining}.

\textbf{(iii)} The committed policy is given by 
Def.~\ref{def:committed_policy}.

Then, the committed policy $\pi_k^{\mathrm{com}}$
satisfies the deterministic safety condition \eqref{eq:deterministic safety condition} in
Def.~\ref{def:safety verified} at each iteration $k$, i.e.,: $
x(t)\in\Scal, u(t)\in\Ucal, \forall t\in [t_k, t_{k+1}], \forall k\geq 0$.
Furthermore, the cumulative
predicted exploration cost satisfies
\begin{equation*}
\sum_{k=0}^{\infty}
\Delta\mathcal J_{\mathrm{exp}}^k
\le B_{\mathrm{exp}}.
\end{equation*}
\end{theorem}
\end{highlight_blue}

\begin{highlight_blue_proof}
\begin{proof}
We proceed by induction over the replanning times.
At $t_0$, assumption (i) provides a feasible backup policy
whose tube contains $x_0$, and
$\mathcal J_{\mathrm{exec}}^0=0$.

Suppose that at $t_k$, a feasible backup policy $\pi_k^B$
is available, $x_k\in\Omega_k^B(t_k)$, and
$\mathcal J_{\mathrm{exec}}^k\leq B_{\mathrm{exp}}$.
By \cref{lem:SMID_consistency}, the true parameter lies in $\Theta_k$. Together with \cref{asm:noise_derivatives}, this ensures that the true parameter and disturbance
satisfy the uncertainty bounds considered in
\cref{eq:quad_backup_invariance,eq:quad_info_safety}. By \eqref{eq:quad_candidate_backup}, every candidate backup policy coincides with $\pi_k^B$. The candidate backup policies satisfy the robust safety conditions by
\eqref{eq:quad_backup_invariance} and the
tightened constraints in \eqref{prob:quad_backup}; thus $\mathcal I_k^B$ is nonempty and $i_B=1$.
Def.~\ref{def:committed_policy} thus selects either
(i) the backup segment $\pi_{k,1}^B$ or
(ii) an informative policy $\pi_{k,i^\star}^I$
from a valid and budget-feasible pair.

In case (i), the executed segment follows $\pi_k^B$.
By \eqref{eq:quad_backup_invariance}, the actual state
remains in $\Omega_k^B(t)$ throughout the committed
interval. The tightened constraints in \eqref{prob:quad_backup} then ensure state and input constraint satisfaction for all $\theta\in\Theta_k$ and disturbances $w$ satisfying
$w(t)\in\Wcal$ for all $t\in[t_k,t_{k+1}]$.

In case (ii), the informative feedback policy maintains
the actual state within $\Omega_{k,i^\star}^I(t)$.
Condition \eqref{eq:quad_info_safety} therefore ensures
state and input constraint satisfaction throughout the
committed interval for all admissible realizations.
If the segment is completed, then
$t_{k+1}=t_{k,i^\star}^s$, and
\eqref{eq:quad_tube_rejoining} gives
$x(t_{k+1})\in\Omega_{k,i^\star}^I(t_{k+1})
\subseteq\Omega_k^B(t_{k+1})$.
Thus, $\pi_k^B$ can safely be applied from the reached
state, with \eqref{eq:quad_backup_invariance} ensuring
constraint satisfaction over the remaining backup horizon.

In either case, the committed interval is safe.
If the mission continues, the reached state lies in
the existing backup tube, and assumption (i) provides
a feasible backup at the next replanning time whose
tube contains that state. This completes the induction
for safety throughout the mission.

For budget feasibility, an informative policy is
selected only if
$\mathcal J_{\mathrm{exec}}^k
+\Delta\mathcal J_{\mathrm{exp}}^k(i^\star)
\leq B_{\mathrm{exp}}$, by \eqref{eq:feasible-set}.
If a backup policy is selected instead, its exploration-budget debit is zero by the budget-update rule in \cref{subsec:commitment}.
Consequently,
$\mathcal J_{\mathrm{exec}}^{k+1}
=\mathcal J_{\mathrm{exec}}^k
+\Delta\mathcal J_{\mathrm{exp}}^k
\leq B_{\mathrm{exp}}$.
Starting from $\mathcal J_{\mathrm{exec}}^0=0$,
this bounds every partial sum of the exploration-cost
increments. Since these increments are nonnegative,
their cumulative sum satisfies
$\sum_{k=0}^{\infty}\Delta\mathcal J_{\mathrm{exp}}^k
\leq B_{\mathrm{exp}}$.
\end{proof}
\end{highlight_blue_proof}

\subsection{Case Study: Quadrotor Navigation with Uncertain Drag}

We evaluate the deterministic instantiation on
goal-directed quadrotor navigation with an unknown
aerodynamic drag coefficient. The quadrotor is
modeled as a double integrator with nonlinear drag,
\begin{equation}
\label{eq:drag_dynamics}
\ddot r
=
-C_d\|\dot r\|\dot r+g+u+d,
\end{equation}
where $r\in\R^3$ is the position,
$g\in\R^3$ is gravitational acceleration,
$u\in\R^3$ is the control input,
$d\in\R^3$ is an additive disturbance, and
$C_d\in\R$ is the unknown drag coefficient.
The state $x=[r^\top,\dot r^\top]^\top$ is fully observed,
as assumed in \cref{asm:full_state}.

The initial parameter uncertainty interval is
$C_d\in[0,1]$, which is updated using SMID
during mission execution.

We consider three navigation scenarios with
different obstacle configurations, illustrated
in Fig.~\ref{fig:car_drone_scenarios}.
The figure shows representative informative
and mission trajectories for the three scenarios.

We compare four methods:
(i) No learning,
(ii) Passive,
(iii) Weighted information, and
(iv) Proposed \dualgatekeeper{}.
The results are summarized in
Table~\ref{tab:quad_scalar_comparison}.
The mission cost for each method is normalized
relative to the no-learning baseline in the
corresponding scenario.

\begin{table}[!tbp]
\centering
\footnotesize
\setlength{\tabcolsep}{3pt}
\renewcommand{\arraystretch}{1.05}
\resizebox{\columnwidth}{!}{%
\begin{tabular}{@{}lrrrr@{}}
\toprule
\textbf{Method} &
\makecell{\textbf{Arrival}\\\textbf{(s)}} &
\makecell{\textbf{Cost}\\\textbf{(\%)}} &
\makecell{\textbf{Reduction}\\\textbf{(\%)}} &
\makecell{\textbf{Budget}\\\textbf{spent/cap}} \\
\midrule

\multicolumn{5}{@{}l}{\textit{Scenario 1}} \\
No learning & 18.68 & 100.00 & 0.00 & -- \\
Passive & 10.78 & 70.85 & 98.33 & -- \\
Weighted information & 13.98 & 91.26 & 99.93 & -- \\
\textbf{Proposed} & \textbf{9.96} & \textbf{64.13}
& 99.08 & 0.00/16.88 \\
\midrule

\multicolumn{5}{@{}l}{\textit{Scenario 2}} \\
No learning & 16.00 & 100.00 & 0.00 & -- \\
Passive & 11.76 & 78.20 & 99.24 & -- \\
Weighted information & 11.96 & 81.52 & 99.89 & -- \\
\textbf{Proposed} & \textbf{11.64} & \textbf{77.72}
& 99.30 & 0.81/12.98 \\
\midrule

\multicolumn{5}{@{}l}{\textit{Scenario 3}} \\
No learning & 15.94 & 100.00 & 0.00 & -- \\
Passive & 11.70 & 78.24 & 99.80 & -- \\
Weighted information & 11.92 & 80.84 & 99.85 & -- \\
\textbf{Proposed} & \textbf{11.54} & \textbf{77.19}
& 99.81 & 0.08/13.87 \\
\bottomrule
\end{tabular}%
}
\caption{Performance comparison across three quadrotor
navigation scenarios. Costs are normalized relative to
the no-learning baseline (100\%). Uncertainty reduction
is $100(1-\text{final width}/\text{initial width})$.
Budget entries report accumulated predicted exploration
cost and the prescribed cap in unnormalized mission-cost units.}
\label{tab:quad_scalar_comparison}
\vspace{-13pt}
\end{table}

\subsubsection{Results}

The proposed method achieves the lowest arrival time and mission cost across all three scenarios. The arrival times are $9.96\,\mathrm{s}$, $11.64\,\mathrm{s}$, and $11.54\,\mathrm{s}$, with corresponding mission costs of $64.13\%$, $77.72\%$, and $77.19\%$ of the no-learning baseline.

The final SMID intervals correspond to uncertainty reductions of $99.08\%$, $99.30\%$, and $99.81\%$, respectively. The weighted-information method achieves slightly greater uncertainty reduction in each scenario but incurs higher mission costs and longer arrival times.

For Scenarios 1--3, the exploration budget usage
(spent/cap) is $0.000/16.876$, $0.808/12.976$, and
$0.082/13.873$, respectively. The spent amount is the
cumulative predicted exploration cost, computed relative
to the paired backup policies using \eqref{eq:exp-cost-local}.
The cap is the prescribed budget $B_{\mathrm{exp}}$
for each scenario. Both are reported in unnormalized
mission-cost units. These results demonstrate that maximizing uncertainty
reduction alone does not necessarily improve mission
performance, whereas \dualgatekeeper{} selects
informative behavior subject to safety and the
prescribed exploration budget.

\section{Probabilistic Instantiation of \dualgatekeeper{}}
\label{sec:gatekeeper_robust}

We consider the probabilistic instantiation of \dualgatekeeper{} for autonomous racing with uncertain tire--road friction. In contrast to the tube-MPC backup, which is safe by
construction, the racing backup follows the nominal
policy before switching to the fallback policy.
Since safety during the nominal segment is not
guaranteed, both candidate policies undergo
sampling-based safety verification
(Fig.~\ref{fig:method_squence_prob}). The probabilistic guarantee arises from evaluating candidate policies under a distribution over uncertain parameters and disturbances, not from stochastic system dynamics. A candidate pair is valid only if both policies satisfy
the probabilistic safety condition in
Def.~\ref{def:safety verified}.

\subsection{Construction of Candidate Policy Pair}

At replanning time $t_k$, consider the nominal mission policy $\pi_k^{\mathrm{nom}}$ from Def.~\ref{def:nominal_policy}, illustrated in Fig.~\ref{fig:method_squence_prob}(b).
Let $\pi^{\mathrm{fb}}:\R\times\Xcal\to\Ucal$
denote a fallback feedback policy, which by definition maintains the trajectories of system~\eqref{eq:system_model} on a fallback-safe set $\Ccal\subseteq\Scal$ for all $\theta\in\Theta_k$ and $w(\cdot)\in\Wcal$ (or in other words, renders the fallback-safe set $\mathcal C$ robustly forward invariant).
In the racing case, $\Ccal$ is the geometric centerline path illustrated in
Fig.~\ref{fig:method_squence_prob}(a), and $\pi^{\mathrm{fb}}$ refers to a policy that tracks this centerline path. Note that reaching the path from conditions outside $\Ccal$ is checked during candidate verification.

 \begin{figure*}[t]
  \centering
  \includegraphics[width=2.05\columnwidth]{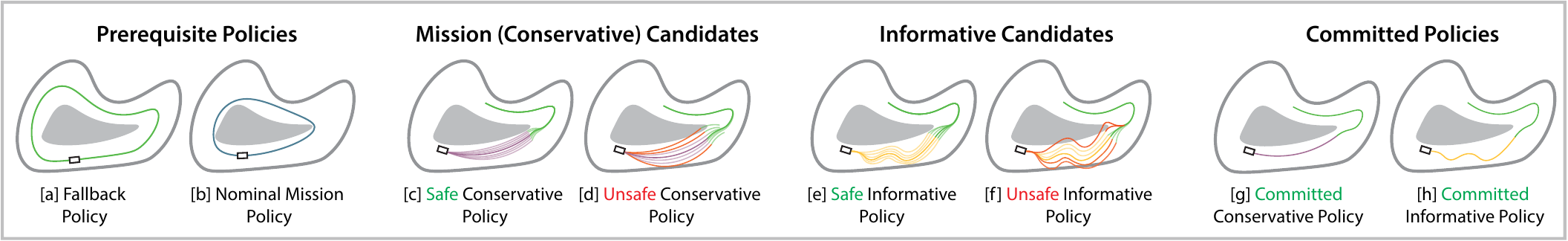}
  \caption{Illustration of the probabilistic instantiation of \dualgatekeeper{}. Conservative and informative candidate policies are formed by combining their respective policy segments with the fallback policy. Candidates are verified for safety before commitment.}
  \vspace{-13pt}
  \label{fig:method_squence_prob}
\end{figure*}

For each candidate interval $\Tcal_i^{c,k}$ from \eqref{eq:set_horizons}, the candidate
backup policy
$\bar{\pi}_{k,i}^{B}:[t_k,\infty)\times\Xcal\to\Ucal$
follows the nominal policy until the switching time
$t_{k,i}^{s}$ and applies the fallback policy $\pi^{\mathrm{fb}}$ thereafter:
\begin{equation}
\label{eq:racing_backup_candidate}
\bar{\pi}_{k,i}^{B}(t,x)=
\begin{cases}
\pi_k^{\mathrm{nom}}(t,x),
& t_k\leq t<t_{k,i}^{s},\\
\pi^{\mathrm{fb}}(t,x),
& t\geq t_{k,i}^{s}.
\end{cases}
\end{equation}

Similarly, let
$\bar{\pi}_{k,i}^{I,\mathrm{seg}}:
\Tcal_i^{c,k}\times\Xcal\to\Ucal$
denote the informative feedback policy segment associated
with the nominal informative trajectory $\bar{p}_{k,i}^{I}$
generated using \eqref{prob:informative_generation}, as
defined in Def.~\ref{def:informative_segment}.
The candidate informative policy
$\bar{\pi}_{k,i}^{I}:[t_k,\infty)\times\Xcal\to\Ucal$
follows this segment until $t_{k,i}^{s}$ and applies
the fallback policy thereafter:
\begin{equation}
\label{eq:racing_informative_candidate}
\bar{\pi}_{k,i}^{I}(t,x)=
\begin{cases}
\bar{\pi}_{k,i}^{I,\mathrm{seg}}(t,x),
& t_k\leq t<t_{k,i}^{s},\\
\pi^{\mathrm{fb}}(t,x),
& t\geq t_{k,i}^{s}.
\end{cases}
\end{equation}


\subsection{Valid Candidate Pair Generation via Probabilistic Safety Verification}

For the candidate pair $(\bar{\pi}_{k,i}^{I},\bar{\pi}_{k,i}^{B})$, both policies are verified under the parameter and disturbance uncertainty. The verification condition considers the current parameter set $\Theta_k$ from SMID
and $w(t)\in\Wcal$ for all $t\ge t_k$, as in Assumption~\ref{asm:noise_derivatives}.

For candidate policy pair $i$, the finite verification interval $[t_k,t_{k,i}^{v}]$ extends beyond the candidate switching time $t_{k,i}^{s}$, with $t_{k,i}^{v}\ge t_{k,i}^{s}$, to include fallback continuation. Let $L_{k,i}$ be the number of simulation steps in this interval and $\mathbf w=(w_0,\ldots,w_{L_{k,i}-1})$ the sampled disturbance sequence.

For probabilistic safety verification, we assume the following joint distribution over the parameter and disturbance samples:
\begin{equation}
\label{eq:racing_uniform}
{\mathbb P}_{k,i}
=
\operatorname{Unif}(\Theta_k)
\otimes
\big(\operatorname{Unif}(\Wcal)\big)^{\otimes L_{k,i}}.
\end{equation}
Under this model, one parameter realization is sampled for the complete rollout, while disturbance values are sampled independently at each simulation step. The assumed distribution is used for verification; it is distinct from the set-containment guarantee of SMID and the disturbance bound.

For each candidate policy $\pi\in\{\bar{\pi}_{k,i}^{B},\bar{\pi}_{k,i}^{I}\}$, verification requires state and input constraint satisfaction at the simulation grid points in $[t_k,t_{k,i}^{v}]$ and arrival in the fallback-safe set $\Ccal$ at $t_{k,i}^{v}$. For a fixed joint realization
$(\theta,\mathbf w)\in\Theta_k\times\Wcal^{L_{k,i}}$,
let $(x_j,u_j)_{j=0}^{L_{k,i}}$ denote the simulated closed-loop state--input sequence from $x_k$ under the candidate policy $\pi$, over the grid spanning $[t_k,t_{k,i}^{v}]$.

The parameter realization remains fixed throughout
the rollout, and each disturbance value $w_j$ is held
constant over simulation step $j$, for
$j=0,\ldots,L_{k,i}-1$.
Define the finite-horizon safety indicator
\begin{equation*}
\label{eq:racing_safety_function}
H_{k,i}(x_k, \pi,\theta,\mathbf w)
=
\mathbb I\left\{
\begin{array}{l}
x_j\in\Scal,\;u_j\in\Ucal,
\;0\leq j\leq L_{k,i},\\
x_{L_{k,i}}\in\Ccal
\end{array}
\right\}.
\end{equation*}
A rollout is considered successful if
$H_{k,i}(x_k, \pi,\theta,\mathbf w)>0$ and a failure otherwise.
Nonfinite rollouts are classified as failures.
Since $\Ccal$ is assumed to be robustly invariant under
$\pi^{\mathrm{fb}}$, entering $\Ccal$ under $\pi^{\mathrm{fb}}$ permits
safe continuation beyond the verification horizon. In this example, this refers to the system trajectories following the slow, centerline trajectory.

The probabilistic safety condition is
\begin{equation}
\label{eq:racing_chance_constraint}
\Pr_{(\theta,\mathbf w)\sim{\mathbb P}_{k,i}}
\big(H_{k,i}(x_k, \pi,\theta,\mathbf w)>0\big)
\geq1-\varepsilon_k,
\end{equation}
where $\varepsilon_k\in(0,1)$ is the failure tolerance. Equivalently, for 
\begin{equation}
\label{eq:racing_failure_probability}
q_{k,i}(\pi)
:=
\Pr_{(\theta,\mathbf w)\sim{\mathbb P}_{k,i}}
\big(H_{k,i}(x_k, \pi,\theta,\mathbf w)\leq0\big),
\end{equation}
the safety requirement becomes
$q_{k,i}(\pi)\leq\varepsilon_k$.
We next establish a finite-sample verification
condition for this requirement.

\subsubsection{Finite-Sample Safety Verification}

Since $q_{k,i}(\pi)$ is generally unavailable
analytically, we verify the chance constraint using
an exact binomial confidence bound~\cite{11704677}.

For each fixed candidate policy
$\pi\in\{\bar{\pi}_{k,i}^{B},\bar{\pi}_{k,i}^{I}\}$,
draw $N$ independent joint realizations
$(\theta^{(\ell)},\mathbf w^{(\ell)})
\sim\hat{\mathbb P}_{k,i}$,
$\ell=1,\ldots,N$.
Thus, each rollout uses one parameter draw and
$L_{k,i}$ independent disturbance draws.
The number of failed closed-loop rollouts is
\begin{equation}
\label{eq:racing_failure_count}
K_{k,i}(\pi)
=
\sum_{\ell=1}^{N}
\mathbb I\left\{
H_{k,i}(x_k, \pi,\theta^{(\ell)},\mathbf w^{(\ell)})\leq0
\right\}.
\end{equation}

Let $\delta_k\in(0,1)$ denote the prescribed
statistical error probability for the complete
replanning iteration.
Since $N_k$ candidate pairs yield at most $2N_k$
individual policy tests, the \v{S}id\'ak correction
for independent tests permits any per-test error
probability satisfying
\begin{equation}
\label{eq:racing_test_error}
0<\rho_k
\leq
1-(1-\delta_k)^{1/(2N_k)}.
\end{equation}

For each candidate policy, define the one-sided
binomial upper confidence bound
\begin{equation}
\label{eq:racing_upper_bound}
\overline q_{k,i}(\pi)
=
\max\left\{
q\in[0,1]:
\operatorname{Bin}(K_{k,i}(\pi);N,q)
\geq\rho_k
\right\},
\end{equation}
where $\operatorname{Bin}(K;N,q)$ denotes the
cumulative probability of at most $K$ failures
in $N$ independent Bernoulli trials with failure
probability $q$.

For each fixed candidate policy, this bound satisfies
\begin{equation}
\label{eq:racing_individual_guarantee}
\Pr_{\mathrm{cert}}\left(
q_{k,i}(\pi)\leq\overline q_{k,i}(\pi)
\right)
\geq1-\rho_k,
\end{equation}
where $\Pr_{\mathrm{cert}}$ denotes probability
over the $N$ independent joint samples drawn from
$\mathbb P_{k,i}$ in \eqref{eq:racing_uniform}.

\begin{table}[!tbp]
\centering
\footnotesize
\setlength{\tabcolsep}{2pt}
\renewcommand{\arraystretch}{1.10}
\begin{tabular}{@{}lrrrr@{}}
\toprule
\textbf{Method} & \textbf{Lap (s)} & \textbf{Cost (\%)} &
\textbf{Reduc. (\%)} & \textbf{Budget}\\
\midrule
No learning & 244.85 & 100.00 & 0.00 & --\\
Passive & 77.10 & 14.85 & 99.45 & --\\
Weighted ($\gamma=0.4$) & 77.90 & 15.26 & 99.97 & --\\
\textbf{Proposed} & \textbf{76.40} & \textbf{14.53} &
99.95 & 0.66/8\\
\bottomrule
\end{tabular}
\caption{Autonomous racing performance at $\varepsilon=5\%$. Cost is normalized to the no-learning baseline. Reduction (\%) is $100(1-\text{final SMID interval width}/\text{initial width})$; budget reports accumulated predicted exploration cost/prescribed cap in unnormalized mission-cost units.}
\label{tab:racing_main}
\vspace{-10pt}
\end{table}

A candidate policy passes safety verification if
\begin{equation}
\label{eq:racing_acceptance}
\overline q_{k,i}(\pi)\leq\varepsilon_k.
\end{equation}
Here, $\varepsilon_k$ specifies the admissible
failure probability, while $\rho_k$ determines
the statistical confidence of the test. The candidate pair
$(\bar{\pi}_{k,i}^{I},\bar{\pi}_{k,i}^{B})$
is valid if
\begin{equation}
\label{eq:racing_pair_validity}
\begin{aligned}
\overline q_{k,i}(\bar{\pi}_{k,i}^{B})
&\leq\varepsilon_k,\\
\overline q_{k,i}(\bar{\pi}_{k,i}^{I})
&\leq\varepsilon_k.
\end{aligned}
\end{equation}

Let $\mathcal I_k^{\mathrm{val}}
\subseteq\{1,\ldots,N_k\}$ denote the indices
satisfying \eqref{eq:racing_pair_validity}.
For each $i\in\mathcal I_k^{\mathrm{val}}$,
the corresponding valid policies are denoted by
$\pi_{k,i}^{B}:=\bar{\pi}_{k,i}^{B}$ and
$\pi_{k,i}^{I}:=\bar{\pi}_{k,i}^{I}$.
The resulting valid candidate-pair set is
$\Vcal_k^c
=
\left\{
(\pi_{k,i}^{I},\pi_{k,i}^{B})
:
i\in\mathcal I_k^{\mathrm{val}}
\right\}$.
Only these pairs are considered for budget
feasibility and uncertainty-reduction-based selection.

\subsection{Probabilistic Safety Guarantee}
We establish a finite-sample confidence guarantee for each valid candidate policy pair at a fixed replanning iteration.

\begin{highlight_blue}
\begin{theorem}[Per-Iteration Probabilistic Safety]
\label{thm:racing_probabilistic}

Consider the system \eqref{eq:system_model}
at replanning time $t_k$, with state $x_k$
and uncertainty set $\Theta_k$.
Suppose that:

\textbf{(i)} For each candidate policy at replanning time $t_k$, safety verification uses independent samples drawn from the distribution
\eqref{eq:racing_uniform}.

\textbf{(ii)} The fallback policy $\pi^{\mathrm{fb}}$
renders $\Ccal$ robustly forward invariant
while satisfying the prescribed input constraints
for all $\theta\in\Theta_k$ and $w(t)\in\Wcal$, $\forall t\ge t_k$.

\textbf{(iii)} The candidate policies are constructed
independently of the verification samples, and
independent sets of $N$ joint samples are used for
all candidate-policy tests.

Let the $N_k$ candidate pairs be verified using
\eqref{eq:racing_failure_count}--
\eqref{eq:racing_pair_validity}, with $\rho_k$
satisfying \eqref{eq:racing_test_error}.
Then,
\begin{equation}
\label{eq:racing_guarantee}
\begin{aligned}
\Pr_{\mathrm{cert}}\Big(
&q_{k,i}(\pi_{k,i}^{\sigma})
\leq\varepsilon_k,\\
&\forall i\in\mathcal I_k^{\mathrm{val}},
\quad \sigma\in\{B,I\}
\Big)
\geq1-\delta_k.
\end{aligned}
\end{equation}

Thus, with confidence at least $1-\delta_k$,
both policies in every valid candidate pair
satisfy the prescribed finite-horizon chance
constraint \eqref{eq:racing_chance_constraint}.
\end{theorem}
\end{highlight_blue}

\begin{table}[!tbp]
\centering
\footnotesize
\setlength{\tabcolsep}{2pt}
\renewcommand{\arraystretch}{1.10}
\begin{tabular}{@{}lrrrr@{}}
\toprule
\textbf{Method} & \textbf{Lap (s)} & \textbf{Cost $J$} &
\textbf{Reduc. (\%)} & \textbf{Info. (s)}\\
\midrule
\textbf{Proposed} & \textbf{76.40} & \textbf{341.71} & 99.95 & 10.50\\
Weighted ($\gamma=0$) & 244.85 & 2351.75 & 0.00 & 0.00\\
Weighted ($\gamma=0.1$) & 80.40 & 386.79 & 99.65 & 73.80\\
Weighted ($\gamma=0.4$) & 77.90 & 358.94 & 99.97 & 71.00\\
Weighted ($\gamma=2$) & 81.55 & 396.96 & 99.99 & 76.45\\
\bottomrule
\end{tabular}
\caption{Information-weight sensitivity at $\varepsilon=5\%$. Red. (\%) denotes the relative decrease in the SMID friction-parameter interval width; Info. (s) is the total duration of informative trajectory.}
\label{tab:racing_gamma}
\vspace{-13pt}
\end{table}

\begin{highlight_blue_proof}
\begin{proof}
Fix the current planning data and candidate policies.
For each candidate
$\pi\in\{\bar{\pi}_{k,i}^{B},\bar{\pi}_{k,i}^{I}\}$,
the $N$ independent joint samples from
\eqref{eq:racing_uniform} yield independent Bernoulli
failure indicators in \eqref{eq:racing_failure_count},
each with failure probability $q_{k,i}(\pi)$.
Thus, the binomial upper bound satisfies
$\Pr_{\mathrm{cert}}\big(
q_{k,i}(\pi)\leq\overline q_{k,i}(\pi)
\big)\geq1-\rho_k$, as stated in
\eqref{eq:racing_individual_guarantee}.

Let $\mathcal E_k$ denote the event that all computed
upper bounds cover their respective true failure
probabilities. Since at most $2N_k$ policy tests are
conducted using independent verification samples,
the individual coverage guarantees imply
\begin{equation}
\label{eq:racing_joint_confidence}
\Pr_{\mathrm{cert}}(\mathcal E_k)
\geq(1-\rho_k)^{2N_k}
\geq1-\delta_k,
\end{equation}
where the final inequality follows from
\eqref{eq:racing_test_error}.

On $\mathcal E_k$, every tested policy satisfies
$q_{k,i}(\pi)\leq\overline q_{k,i}(\pi)$.
A policy is accepted only if
$\overline q_{k,i}(\pi)\leq\varepsilon_k$;
therefore, every accepted policy satisfies
$q_{k,i}(\pi)\leq\overline q_{k,i}(\pi)\leq\varepsilon_k$.
This applies to both policies in every valid pair
$i\in\mathcal I_k^{\mathrm{val}}$ and to every
individually accepted backup policy.

Consequently, with confidence at least $1-\delta_k$,
all these policies satisfy the prescribed failure
tolerance, establishing \eqref{eq:racing_guarantee}.
Because this guarantee holds simultaneously, selecting
an informative policy from a valid and budget-feasible
pair, or an individually accepted backup policy,
preserves the per-iteration certificate.
\end{proof}
\end{highlight_blue_proof}

\begin{remark}[Scope of the safety guarantee]
The certificate applies to the discrete verification grid
at one replanning iteration; it does not establish intersample
safety or a joint guarantee across iterations.
Once the actual state reaches $\Ccal$, robust invariance
under $\pi^{\mathrm{fb}}$ ensures safe continuation.
If $\pi_{k-1}^{\mathrm{com}}$ is continued, its original
switching schedule is preserved. Its certificate remains
tied to the original verification state and distribution;
a bound under the updated distribution requires reverification.
\end{remark}

\subsection{Autonomous Racing Case Study}

We evaluate the probabilistic instantiation on a technical
racing circuit using a seven-state dynamic bicycle model with
an unknown tire--road friction coefficient $\mu$. The initial uncertainty
interval is $\Theta_0=[0.45,1.05]$, which is refined through
SMID using measurements collected during execution.

We compare four methods: No learning, which keeps $\Theta_0$
fixed; Passive, which updates SMID during nominal execution;
Weighted information, which incorporates an information term
directly into MPPI; and Proposed \dualgatekeeper{}, which
evaluates informative candidates subject to safety and budget
feasibility. The prescribed failure
tolerance is $\varepsilon=5\%$. We use $N=1000$ verification samples per policy,
a maximum of $M=64$ candidate-policy tests per
replanning iteration, and $\delta_k=\delta=0.05$.
The experiments use the conservative per-test
allocation $\rho_k=\rho=\delta/M$, which satisfies
\eqref{eq:racing_test_error} whenever $2N_k\leq M$.

\subsubsection{Lap Performance}

Table~\ref{tab:racing_main} reports mission costs
normalized by the no-learning baseline.
Representative racing trajectories are shown in the
right panel of Fig.~\ref{fig:car_drone_scenarios}.
The proposed method achieves the lowest lap time ($76.40\,\mathrm{s}$) and mission cost ($14.53\%$ of the no-learning baseline, or $341.71$ unnormalized),
while reducing the SMID friction-parameter interval width by $99.95\%$. Its accumulated predicted exploration cost is $0.663$ against the prescribed budget
$B_{\mathrm{exp}}=8$ (reported as $0.66/8$ in Table~\ref{tab:racing_main}), computed relative to paired backup policies using \eqref{eq:exp-cost-local}
and the budget update in \cref{subsec:commitment}. Although the weighted method with $\gamma=0.4$ achieves slightly greater uncertainty reduction, it incurs higher mission cost and lap time.

\subsubsection{Effect of Information Weight}

Table~\ref{tab:racing_gamma} compares the proposed method
with weighted-information MPPI across different information
weights, using $\varepsilon=5\%$ throughout.
 The proposed method executes informative controls for $10.50\,\mathrm{s}$, compared with $71.00$--$76.45\,\mathrm{s}$ for positive information weights. Except for the $\gamma=0$ case, all methods substantially reduce the SMID friction-parameter interval width. Reductions near $100\%$ indicate that the final interval is much narrower than the initial one, rather than exact knowledge of the true parameter. Despite achieving the largest width reduction ($99.99\%$), the weighted method with $\gamma=2$ has longer lap time and higher mission cost than our method.

\subsubsection{Effect of Failure Tolerance}

We examine the effect of the prescribed failure tolerance
through matched low-friction tests with $\mu=0.45$ and
$\Theta_0=[0.45,1.05]$. At $t=0.30\,\mathrm{s}$, SMID
yields $\Theta_k=[0.45,0.578775]$. Candidate generation
and verification samples are held fixed, varying only
$\varepsilon$.

\begin{table}[!tbp]
\centering
\footnotesize
\setlength{\tabcolsep}{3pt}
\renewcommand{\arraystretch}{1.10}
\begin{tabular*}{\columnwidth}{@{\extracolsep{\fill}}lccc@{}}
\toprule
$\varepsilon$ & \textbf{Failures/$N$} &
\textbf{Upper (\%)} & \textbf{Execution outcome}\\
\midrule
1\% & 0/1000 & 0.713 & No violation\\
5\% & 22/1000 & 4.081 & Tire-force violation\\
10\% & 66/1000 & 9.447 & Tire-force violation\\
\bottomrule
\end{tabular*}
\caption{Selected-policy outcomes in matched low-friction tests. Failures/$N$ is the number of failed verification rollouts; Upper (\%) is the one-sided binomial upper bound on failure probability. Execution outcomes are assessed over eight seconds.}
\label{tab:racing_risk}
\vspace{-13pt}
\end{table}

As shown in Table~\ref{tab:racing_risk}, the $0.30\,\mathrm{s}$
informative candidate has an upper failure bound of $4.081\%$
and passes verification at $\varepsilon=5\%$ and $10\%$,
but not at $1\%$. The $1.20\,\mathrm{s}$ informative
candidate has an upper bound of $9.447\%$ and passes
only at $10\%$. The latter executions exhibit tire-force constraint
violations, whereas the $1\%$ execution remains
violation-free over the eight-second evaluation horizon. These outcomes are consistent with the nonzero allowable failure probabilities and do not establish whole-lap failure rates.

\noindent\textbf{Computational setup.} All simulations in both case studies were run on Ubuntu 24.04.3 LTS with an Intel Core i7-13650HX, 32 GB RAM, and an NVIDIA RTX 4060 GPU.

\section{Conclusion}
We presented a dual-control framework integrating safety, budget feasibility,
and active uncertainty reduction. By treating informative trajectories as
certifiable decisions and committing only those satisfying safety and cost
constraints, the framework reduces parametric uncertainty while ensuring
reliable execution. The case studies demonstrate lower mission cost and tighter
uncertainty bounds than conservative baselines. Future work will focus on
incorporating safety directly into informative candidate generation, quantifying
how uncertainty reduction affects future mission cost, and improving the
fallback policy and reachable fallback set to reduce conservativeness.

\nocite{*}
\bibliographystyle{IEEEtran}
\bibliography{main.bib}


\begingroup
\appendix

\subsection{Interpretation of the Predicted Width}
\label{appendix:uncertainty_reduction_anaylis}

The two approaches in \cref{sec:uncertainty_pred} give different
interpretations of $\hat w_{d,k,i}$. Approach~1 uses the empirical mean
of the widths of synthetic SMID sets,
\begin{equation}
\label{eq:app_simulated_width}
\hat w_{d,k,i}
=
\frac{1}{N}\sum_{\ell=1}^{N}
w_d\big(\Theta_{k+1,i}^{(\ell)}\big).
\end{equation}
It estimates the expected width under the parameter and disturbance
sampling model specified in Approach~1. It is not a deterministic
upper bound on the width obtained during execution.

Approach~2 instead bounds the width of the set consistent with future
data, conditional on agreement between the planned and executed
integrated regressors. We make this interpretation precise below.
Fix a replanning time $t_k$ and candidate $i$ over
$\Tcal_i^{c,k}=[t_k,t_k+T_i^{c,k}]$. Write
$A_{k,i}\in\R^{M\times p}$ for the matrix $A$ in \eqref{eq:Astack}
formed from its $N_j$ planned integrated regressors
$\mathcal F_j\in\R^{n\times p}$, where $M=N_jn$.
Each regressor uses the window $[\tau_j-\Delta,\tau_j]$,
including recorded data for any portion preceding $t_k$.
The candidate error set is
\begin{equation}
\label{eq:app_error_set}
\Ecal_{\theta,k,i}
=
\big\{e_\theta\in\R^p:
\|A_{k,i}e_\theta\|_\infty\leq2\epsilon\big\},
\end{equation}
where $\epsilon$ is the same regression-error bound used by SMID
in \cref{sec:smid}, including the integrated disturbance and
numerical integration error.

\begin{assumption}[Agreement of prediction and identification data]
\label{assump:exec=plan}
The integration windows used to predict candidate $i$'s uncertainty
reduction are realized during execution. Their executed integrated
regressors $\mathcal F_j^{\mathrm{act}}\in\R^{n\times p}$ satisfy
$\mathcal F_j^{\mathrm{act}}=\mathcal F_j$, $j=1,\ldots,N_j$.
The corresponding realized outputs obey
\[
Y_j=\mathcal F_j^{\mathrm{act}}\theta^\star+e_j,
\qquad \|e_j\|_\infty\leq\epsilon,
\]
where $\theta^\star\in\Theta_k$ is the true parameter.
The SMID update retains these regression pairs and uses the same
bound $\epsilon$.
\end{assumption}

Equality of the planned and executed state--input histories over
these windows is sufficient for this assumption; equality of the
integrated regressors is the condition needed for the result.
It is an additional condition for interpreting the prediction,
not an assumption needed for the main safety or budget results.

Let $A_{k,i}^{\mathrm{act}}\in\R^{M\times p}$ stack the executed
regressors, and let $Y\in\R^M$ and $e\in\R^M$ stack the realized
outputs and regression errors, respectively. Then
\[
Y=A_{k,i}^{\mathrm{act}}\theta^\star+e,
\qquad \|e\|_\infty\leq\epsilon.
\]
Define the exact set consistent with these data by
\begin{equation}
\label{eq:app_actual_feasible}
\begin{aligned}
\Pcal_{k,i}^{\mathrm{act}}
:=
\big\{\theta\in\Theta_k:\,
\|Y-A_{k,i}^{\mathrm{act}}\theta\|_\infty
\leq\epsilon\big\}.
\end{aligned}
\end{equation}
This set is nonempty because it contains $\theta^\star$.

\begin{prop}[Conditional upper bound on the realized width]
\label{prop:pred_upper_bound}
Under \cref{assump:exec=plan}, the exact consistent set satisfies
\begin{equation}
\label{eq:app_feasible_inclusion}
\Pcal_{k,i}^{\mathrm{act}}
\subseteq
\Theta_k\cap\big(\theta^\star+\Ecal_{\theta,k,i}\big).
\end{equation}
Consequently, for every $d\in\Dcal$,
\begin{equation}
\label{eq:app_width_bound}
\begin{aligned}
w_d\big(\Pcal_{k,i}^{\mathrm{act}}\big)
&\leq
w_d\big(\Theta_k\cap
(\theta^\star+\Ecal_{\theta,k,i})\big)\\
&\leq
\min\big\{w_d(\Theta_k),
2h_{\Ecal_{\theta,k,i}}(d)\big\}\\
&=\hat w_{d,k,i},
\end{aligned}
\end{equation}
where $\hat w_{d,k,i}$ is the Approach~2 prediction.
\end{prop}

\begin{proof}
For any $\theta\in\Pcal_{k,i}^{\mathrm{act}}$,
\cref{assump:exec=plan} gives $A_{k,i}^{\mathrm{act}}=A_{k,i}$ and
\begin{align*}
\|A_{k,i}(\theta-\theta^\star)\|_\infty
&\leq
\|A_{k,i}\theta-Y\|_\infty
+\|Y-A_{k,i}\theta^\star\|_\infty\\
&\leq2\epsilon.
\end{align*}
Thus $\theta-\theta^\star\in\Ecal_{\theta,k,i}$, proving
\eqref{eq:app_feasible_inclusion}. Width is monotone under set
inclusion and invariant under translation, as shown below.
Hence the width of the intersection is bounded by both
$w_d(\Theta_k)$ and $w_d(\Ecal_{\theta,k,i})$.
By \cref{lem:width-2h},
$w_d(\Ecal_{\theta,k,i})=2h_{\Ecal_{\theta,k,i}}(d)$,
which establishes \eqref{eq:app_width_bound}.
\end{proof}

The intersection in \eqref{eq:app_feasible_inclusion} is an outer
bound over admissible regression errors, not the SMID update for
one realized data set. The support function is evaluated using
\eqref{eq:prop_des_form} with $A=A_{k,i}$.
If $d\notin\operatorname{range}(A_{k,i}^{\top})$, it is infinite,
and the predictor returns the prior width $w_d(\Theta_k)$.

\subsubsection{Relation to hyperrectangular SMID}

For a nonempty compact set $\mathcal Q\subset\R^p$, let
$\operatorname{box}(\mathcal Q)$ denote its smallest enclosing
axis-aligned hyperrectangle. SMID in \eqref{eq:smid_update}
stores such a box. Taking this box preserves coordinate widths:
\begin{equation}
\label{eq:app_box_coordinate}
w_{e_r}\big(\operatorname{box}(\mathcal Q)\big)
=
w_{e_r}(\mathcal Q),
\qquad r=1,\ldots,p,
\end{equation}
where $e_r\in\R^p$ is the $r$-th standard basis vector.
Therefore, when the update uses the data in
\eqref{eq:app_actual_feasible},
\[
w_{e_r}\big(\operatorname{box}(\Pcal_{k,i}^{\mathrm{act}})\big)
\leq\hat w_{e_r,k,i}.
\]
Retaining additional valid regression constraints, or using a
smaller prior contained in $\Theta_k$, can only decrease these
coordinate widths, provided the predicted constraints are
retained. Thus, for $\Dcal=\{e_1,\ldots,e_p\}$, Approach~2
upper-bounds the updated SMID widths and lower-bounds the
corresponding average width reduction under
\cref{assump:exec=plan}.

A box can have larger widths than $\mathcal Q$ in other directions.
For $d=[d_1,\ldots,d_p]^\top\in\R^p$, the coordinate bounds give
\begin{equation}
\label{eq:app_box_direction}
\begin{aligned}
w_d\big(\operatorname{box}(\Pcal_{k,i}^{\mathrm{act}})\big)
&=
\sum_{r=1}^{p}|d_r|\,
w_{e_r}(\Pcal_{k,i}^{\mathrm{act}})\\
&\leq
\sum_{r=1}^{p}|d_r|\,\hat w_{e_r,k,i}.
\end{aligned}
\end{equation}
The original $\hat w_{d,k,i}$ bound therefore applies directly
to the exact consistent set; applying it to the SMID box in
noncoordinate directions requires accounting for this box
approximation. If execution changes the integrated regressors,
the conditional bound need not hold. Both approaches supply
candidate scores; safety verification and the cumulative
predicted exploration-cost budget are enforced separately
as described in the main paper.

\subsection{Useful Properties of Directional Width}
\label{appendix:width_properties}

Let $\mathcal Q\subset\R^p$ be nonempty and compact, let
$a\in\R^p$, and let $d\in\R^p$. Translation invariance follows
from cancellation of the common shift:
\begin{align*}
w_d(\mathcal Q+a)
&=
\sup_{x\in\mathcal Q}d^\top(x+a)
-\inf_{x\in\mathcal Q}d^\top(x+a)\\
&=
\sup_{x\in\mathcal Q}d^\top x+d^\top a
-\inf_{x\in\mathcal Q}d^\top x-d^\top a\\
&=w_d(\mathcal Q).
\end{align*}
If $\mathcal Q_1\subseteq\mathcal Q_2$ are nonempty sets with finite
directional widths, then
\[
\sup_{x\in\mathcal Q_1}d^\top x\leq
\sup_{x\in\mathcal Q_2}d^\top x,
\qquad
\inf_{x\in\mathcal Q_1}d^\top x\geq
\inf_{x\in\mathcal Q_2}d^\top x,
\]
so $w_d(\mathcal Q_1)\leq w_d(\mathcal Q_2)$.
For the possibly unbounded error set
$\Ecal_{\theta,k,i}$, the same arguments apply with
extended-real widths; an infinite width leaves the finite
prior bound in \eqref{eq:app_width_bound} unchanged.
Finally,
\[
\operatorname{box}(\mathcal Q)
=
\prod_{r=1}^{p}
\left[
\min_{\theta\in\mathcal Q}\theta_r,\,
\max_{\theta\in\mathcal Q}\theta_r
\right]
\]
has the same coordinate extrema as $\mathcal Q$, proving
\eqref{eq:app_box_coordinate}. Maximizing and minimizing
$d^\top\theta$ coordinatewise over this box gives the equality
in \eqref{eq:app_box_direction}.

\subsection{Autonomous Car Racing Dynamics and LIP Form}
\label{sec:appex_car}

\subsubsection{Dynamic bicycle model}
\label{sec:appex_car_dyn}

The racing case study in \cref{sec:gatekeeper_robust} uses a
planar dynamic bicycle model with uncertain tire--road friction.
The state and input are
\begin{equation}
\label{eq:app_racing_state_input}
\begin{aligned}
x&=
[p_x,\ p_y,\ \psi,\ v_x,\ v_y,\ \omega,\ \delta]^\top
\in\R^7,\\
u&=
[F_d,\ F_b,\ \dot\delta_{\mathrm{cmd}}]^\top
\in\R^3.
\end{aligned}
\end{equation}
Here $(p_x,p_y)\in\R^2$ is the global position, $\psi\in\R$
is the yaw angle, $v_x,v_y\in\R$ are the body-frame
longitudinal and lateral velocities, $\omega\in\R$ is the
yaw rate, and $\delta\in\R$ is the steering angle.
The scalar inputs are the total drive force $F_d\geq0$,
the signed braking force $F_b\leq0$, and the commanded steering
rate $\dot\delta_{\mathrm{cmd}}\in\R$.

The model describes forward driving, with $v_x\geq0$.
Let $m_{\mathrm v}>0$ be the vehicle mass, $J_z>0$ the yaw
moment of inertia, $l_f,l_r>0$ the distances from the center of
mass to the front and rear axles, and $l=l_f+l_r$ the wheelbase.
The scalar $g_{\mathrm{grav}}>0$ is gravitational acceleration.
These quantities are known. The global kinematics are
\begin{subequations}
\label{eq:app_racing_kinematics}
\begin{align}
\dot p_x&=v_x\cos\psi-v_y\sin\psi,\\
\dot p_y&=v_x\sin\psi+v_y\cos\psi,\\
\dot\psi&=\omega.
\end{align}
\end{subequations}

The front and rear slip angles are
\begin{equation}
\label{eq:app_slip_angles}
\begin{aligned}
\alpha_f(x)
&=\delta-\tan^{-1}\!\left(
\frac{l_f\omega+v_y}{v_x+\epsilon_{\mathrm v}}\right),\\
\alpha_r(x)
&=\tan^{-1}\!\left(
\frac{l_r\omega-v_y}{v_x+\epsilon_{\mathrm v}}\right),
\end{aligned}
\end{equation}
where $\epsilon_{\mathrm v}>0$ is a small velocity regularizer,
distinct from the regression-error bound $\epsilon$.
Assuming equal left and right loads and neglecting load
transfer, the normal loads \emph{per tire} are
\begin{equation}
\label{eq:app_normal_loads}
F_{z,f}=\frac{m_{\mathrm v}g_{\mathrm{grav}}l_r}{2l},
\qquad
F_{z,r}=\frac{m_{\mathrm v}g_{\mathrm{grav}}l_f}{2l}.
\end{equation}
All tire forces below use this per-tire convention, so the
vehicle equations contain a factor of two for each axle.

The lateral forces follow the simplified Pacejka model
\begin{equation}
\label{eq:app_lateral_forces}
\begin{aligned}
F_{y,f}(x,\mu)&=\mu\,\bar F_{y,f}(x),\\
F_{y,r}(x,\mu)&=\mu\,\bar F_{y,r}(x),\\
\bar F_{y,f}(x)
&=F_{z,f}\sin\!\big(C_f\tan^{-1}(B_f\alpha_f(x))\big),\\
\bar F_{y,r}(x)
&=F_{z,r}\sin\!\big(C_r\tan^{-1}(B_r\alpha_r(x))\big).
\end{aligned}
\end{equation}
Here $B_f,B_r,C_f,C_r>0$ are known tire-shape parameters, and
$\mu\in\Theta_k\subset\R$ is the unknown, constant friction
coefficient. The barred quantities are known force factors.

Let $k_d,k_b\in[0,1]$ be the fractions of drive and braking
force allocated to the front axle, and let $f_r\geq0$ be the
known rolling-resistance coefficient. Define the per-tire
rolling-resistance magnitudes
\[
R_f=\frac{f_rm_{\mathrm v}g_{\mathrm{grav}}l_r}{2l},
\qquad
R_r=\frac{f_rm_{\mathrm v}g_{\mathrm{grav}}l_f}{2l}.
\]
The longitudinal tire forces are
\begin{equation}
\label{eq:app_longitudinal_forces}
\begin{aligned}
F_{x,f}(u)&=\tfrac12 k_dF_d+\tfrac12 k_bF_b-R_f,\\
F_{x,r}(u)&=\tfrac12(1-k_d)F_d+\tfrac12(1-k_b)F_b-R_r.
\end{aligned}
\end{equation}

The aerodynamic drag force is
\[
D_a(v_x)=\tfrac12\rho_{\mathrm{air}}A_{\mathrm{fr}}C_Dv_x^2,
\]
where $\rho_{\mathrm{air}}>0$ is air density,
$A_{\mathrm{fr}}>0$ is frontal area, and $C_D\geq0$ is the
known aerodynamic drag coefficient.
Before adding the bounded disturbance $w(t)$ from
\eqref{eq:system_model}, the body-frame and steering dynamics are
\begin{subequations}
\label{eq:app_racing_body}
\begin{align}
\dot v_x
&=\omega v_y+
\frac{
2F_{x,r}+2F_{x,f}\cos\delta
-2F_{y,f}\sin\delta-D_a(v_x)}
{m_{\mathrm v}},\\
\dot v_y
&=-\omega v_x+
\frac{2F_{y,r}+2F_{y,f}\cos\delta+2F_{x,f}\sin\delta}
{m_{\mathrm v}},\\
\dot\omega
&=
\frac{2l_f(F_{y,f}\cos\delta+F_{x,f}\sin\delta)
-2l_rF_{y,r}}{J_z},\\
\dot\delta&=\dot\delta_{\mathrm{cmd}}.
\end{align}
\end{subequations}

\subsubsection{Linear-in-parameters form}
\label{sec:appex_car_LIP}

Because $\delta$ is a state and all tire-shape and vehicle
parameters are known, the only unknown coefficient is $\mu$.
The complete dynamics specialize \eqref{eq:system_model} to
\begin{equation}
\label{eq:app_racing_lip}
\dot x=f_0(x)+g_0(x)u+\Phi_\mu(x)\mu+w(t),
\end{equation}
where $f_0:\Xcal\to\R^7$ and $g_0:\Xcal\to\R^{7\times3}$ are
known maps, $\Phi_\mu:\Xcal\to\R^{7\times1}$ is the friction
regressor, and $w(t)\in\Wcal\subset\R^7$ is the bounded
additive disturbance. The state domain satisfies
$\Xcal\subseteq\{x\in\R^7:v_x\geq0\}$.
These maps are explicitly
\begin{equation}
\label{eq:app_racing_f0}
f_0(x)=
\begin{bmatrix}
v_x\cos\psi-v_y\sin\psi\\
v_x\sin\psi+v_y\cos\psi\\
\omega\\
\omega v_y-\dfrac{2R_r+2R_f\cos\delta+D_a(v_x)}{m_{\mathrm v}}\\[1mm]
-\omega v_x-\dfrac{2R_f\sin\delta}{m_{\mathrm v}}\\[1mm]
-\dfrac{2l_fR_f\sin\delta}{J_z}\\[1mm]
0
\end{bmatrix},
\end{equation}
\begin{equation}
\label{eq:app_racing_g0}
g_0(x)=
\begin{bmatrix}
0&0&0\\
0&0&0\\
0&0&0\\
\dfrac{1-k_d+k_d\cos\delta}{m_{\mathrm v}}&
\dfrac{1-k_b+k_b\cos\delta}{m_{\mathrm v}}&0\\[1mm]
\dfrac{k_d\sin\delta}{m_{\mathrm v}}&
\dfrac{k_b\sin\delta}{m_{\mathrm v}}&0\\[1mm]
\dfrac{l_fk_d\sin\delta}{J_z}&
\dfrac{l_fk_b\sin\delta}{J_z}&0\\[1mm]
0&0&1
\end{bmatrix},
\end{equation}
and
\begin{equation}
\label{eq:app_racing_regressor}
\Phi_\mu(x)=
\begin{bmatrix}
0\\
0\\
0\\
-\dfrac{2\sin\delta\,\bar F_{y,f}(x)}{m_{\mathrm v}}\\[1mm]
\dfrac{2\big(\bar F_{y,r}(x)+\cos\delta\,\bar F_{y,f}(x)\big)}
{m_{\mathrm v}}\\[1mm]
\dfrac{2\big(l_f\cos\delta\,\bar F_{y,f}(x)
-l_r\bar F_{y,r}(x)\big)}{J_z}\\[1mm]
0
\end{bmatrix}.
\end{equation}
Accordingly, the notation of \cref{sec:system_model,sec:lip_form}
specializes to $n=7$, $m=3$, $p=p_f=1$, $p_g=0$,
$\theta=\theta_f=\mu$, and
$F(x)=\Phi(x,u)=\Phi_\mu(x)$; the uncertain input-map term
$G(x)\theta_g$ is absent. Substituting
\eqref{eq:app_racing_f0}--\eqref{eq:app_racing_regressor} into
\eqref{eq:app_racing_lip} recovers the physical equations above.

Identification uses the derivative-free regression
\eqref{eq:integral_regression}. For each window
$[\tau_j-\Delta,\tau_j]$,
\begin{equation}
\label{eq:app_racing_integrated}
\begin{aligned}
\mathcal F_j
&=\int_{\tau_j-\Delta}^{\tau_j}
\Phi_\mu(x(s))\,ds\in\R^{7\times1},\\
Y_j&=\mathcal F_j\mu+e_j,\qquad
\|e_j\|_\infty\leq\epsilon,
\end{aligned}
\end{equation}
where $Y_j\in\R^7$ is computed from the measured state and input
history using the known maps $f_0$ and $g_0$, and
$e_j\in\R^7$ includes the integrated disturbance and numerical
integration error. SMID updates the scalar interval
$\Theta_k=[\theta_1^{k,-},\theta_1^{k,+}]$ from the initial
$\Theta_0=[0.45,1.05]$ used in the racing case study.
Its width is $\theta_1^{k,+}-\theta_1^{k,-}$.
As specified in \cref{sec:uncertainty_pred}, autonomous racing
uses Approach~1 to predict this interval's reduction.

\endgroup





\end{document}